\documentclass[journal]{IEEEtran}

\usepackage{times}
\usepackage{multicol}
\usepackage[bookmarks = true]{hyperref}
\usepackage[utf8]{inputenc}
\usepackage[style=ieee]{biblatex}

\usepackage[super]{nth} %for 2nd, 3rd, etc.
\usepackage{graphicx}
\usepackage{algorithm}
\usepackage{algpseudocode}
\usepackage{amsthm} %needed for math things, including argmin
\usepackage{amsmath, amssymb}
\usepackage{caption} %I assume for continue float
\newtheorem{theorem}{Theorem} %define theorem
\newtheorem{lemma}{Lemma}

\newcommand{\RNum}[1]{\uppercase\expandafter{\romannumeral #1\relax}}
\usepackage{xcolor}
\theoremstyle{definition}
\newtheorem{definition}{Definition}[]
\usepackage{booktabs}

\begin{document}

\title{Continuous Sculpting: Persistent Swarm\\ Shape Formation Adaptable to \\Local Environmental Changes
%Persistent Sculpting: Continuous 2D Shape Formation with Intuitive Human-Swarm Interaction
%Adaptive, Persistent 2D Shape Formation and Path Planning via Decentralized Swarms
}

\author{Andrew G. Curtis$^{1}$, Mark Yim$^{2}$, Michael Rubenstein$^{1}$ % <-this % stops a space
\thanks{Manuscript received: January 5, 2024.}%; Revised March 9, 2023; Accepted April 12, 2023.}% <-this % stops a space
\thanks{This paper was recommended for publication by Editor TBD.}% <-this % stops a space
\thanks{This work was supported by the NDSEG Fellowship and The National Science Foundation, NRI2.0 grants  2024692 and 2024615.}% <- this % stops a space
\thanks{$^{1}$Andrew G. Curtis and Michael Rubenstein are with the Center for Robotics and Biosystems, McCormick School of Engineering, Northwestern University, Evanston, IL USA {\tt\footnotesize agc@u.northwestern.edu, rubenstein@northwestern.edu}.} 
\thanks{$^{2}$Mark Yim is with the General Robotics, Automation, Sensing and Perception (GRASP) Lab, University of Pennsylvania, Philadelphia, PA USA {\tt\footnotesize yim@seas.upenn.edu}.}
% \thanks{Digital Object Identifier (DOI): see top of this page.}
}%

\maketitle

\begin{abstract}
Despite their growing popularity, swarms of robots remain limited by the operating time of each individual. We present algorithms which allow a human to sculpt a swarm of robots into a shape that persists in space perpetually, independent of onboard energy constraints such as batteries. Robots generate a path through a shape such that robots cycle in and out of the shape. Robots inside the shape react to human initiated changes and adapt the path through the shape accordingly. Robots outside the shape recharge and return to the shape so that the shape can persist indefinitely. The presented algorithms communicate shape changes throughout the swarm using message passing and robot motion. These algorithms enable the swarm to persist through any arbitrary changes to the shape. We describe these algorithms in detail and present their performance in simulation and on a swarm of mobile robots. The result is a swarm behavior more suitable for extended duration, dynamic shape-based tasks in applications such as agriculture and emergency response.
\end{abstract}


% \IEEEpeerreviewmaketitle

%DON"T FORGET TO USE CITET and HYPERLINKS!!

\section{Introduction}
%establish context.
As their popularity continues to grow, commercial drones are employed in swarms more and more frequently. 
%provide specific context to shapes
Swarms of flying robots are commonly used in shape-based applications, such as drone shows, where robots navigate through a set of predefined waypoints to form different shapes.
Other shape-based flying swarm applications include search and rescue, emergency response communications networks, and crop monitoring, where a flying swarm maintains a 2D formation over an area of land to perform its task.

%state the problem.
The challenge with most shape-based flying swarm applications is the  limited flight time, or endurance, of the robots. 
Typically, robots must land to recharge,
and the task is paused until the robots can fly again (e.g., crops go unmonitored while the robots charge). 
In addition, most swarms only form static predefined shapes, %are nonreactive to real-time changes, which diminishes 
limiting their potential applications.%greatly reduces the practical applications of aerial swarm shape formation.

\IEEEpubidadjcol %This is required to make room for the footer at the bottom of page 1.

\begin{figure}[!t]
\centering
\includegraphics[width=88mm]{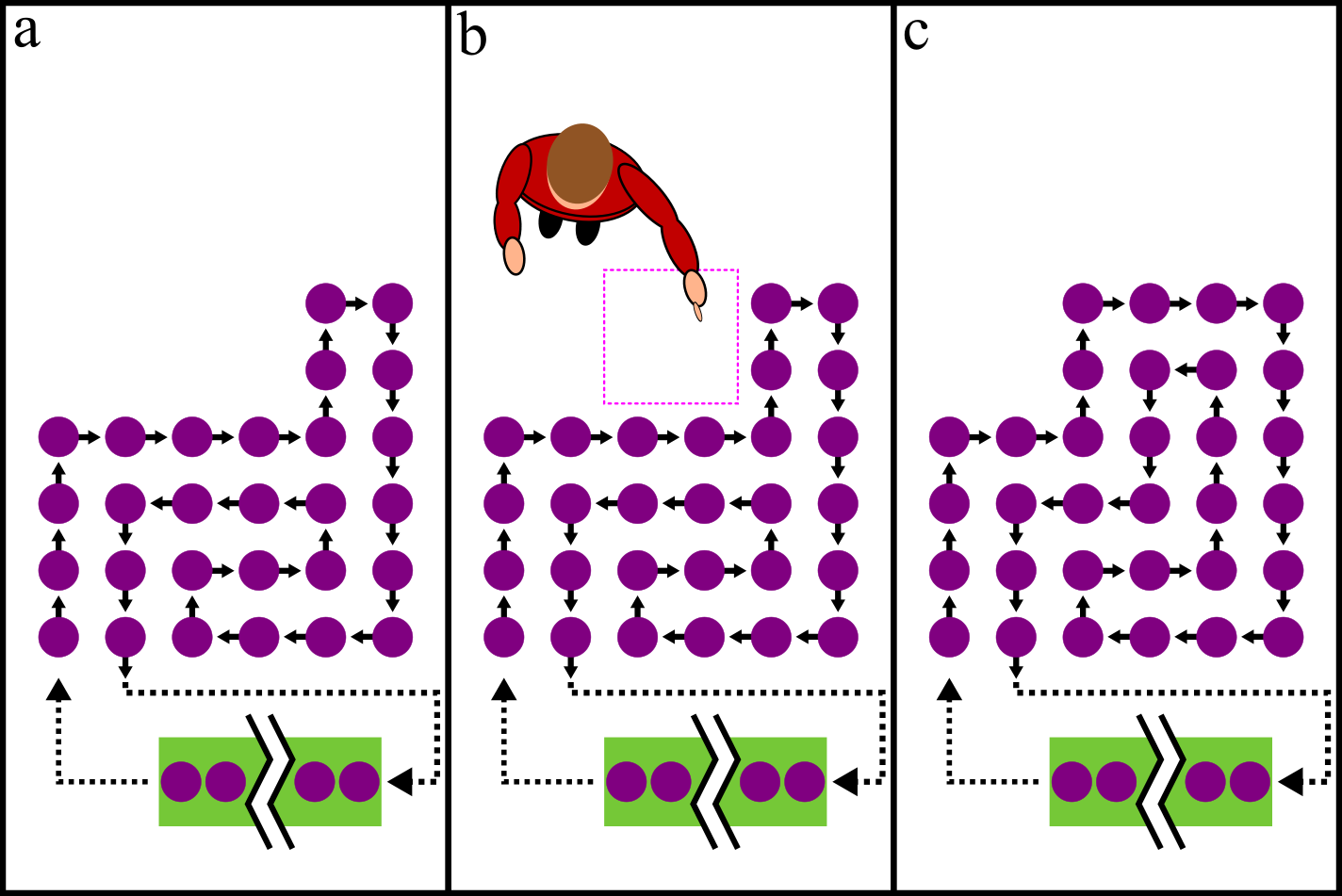}
\caption{Cartoon overhead view of adaptive and persistent shape formation. (a) A swarm of mobile robots (purple circles) leave a charging station (yellow box) to enter and form a shape. The robots move through the shape in the directions indicated by the black arrows until they exit the shape and return to the charging station. (b) A human points to where they would like to add to the shape (purple box). (c) The swarm adjusts to form a path through the new shape while continuing to cycle to and from the charging station.}
\label{ov1}
\end{figure}

%respond to the problem.
To overcome these challenges, we present a novel approach to persistent and adaptable shape formation.
We shift the paradigm of shape formation from shapes formed by static robots to shapes formed by a sequence of robots moving along a path. %a largely unexplored field of study where 
The path allows robots to cycle to and from a charging station, so robot endurance is no longer a constraint on shape duration.
The path is also adaptable to external environmental stimuli (e.g., human interactions) so that the shape is free to change over the duration of the task.
For example, a farmer could both persistently monitor crops (regardless of individual robot charge) and directly interact with the swarm to change the fields of crops being monitored.

We achieve shape persistence by allowing swarm robots to cycle between a shape and a charging station indefinitely via a path that both approximates the shape and facilitates the movement of robots through the shape and back to a charging station (Fig.~\ref{ov1}).
%% These persistent shapes are also adaptable.
With each change to the shape, the swarm adapts to another path and the process continues. 
The algorithms we present are decentralized and scalable to large swarms.
They also open the door to a new application of human-swarm interaction - continuous sculpting - where a human can actively morph a swarm into a persistent shape. % that will exist continuously.
One potential issue with continuous sculpting is making sure potential collisions with humans are safe, but this can be solved by making robots small and light~\cite{curtis2023autonomous}.

The remainder of this paper is structured as follows.
After discussing related work (Section~\ref{sec:related work}), we describe the basic robot capabilities required for a swarm to execute the algorithms (Section~\ref{sec:robot capabilities}).
Then, we introduce the algorithm responsible for shape persistence, called the \textbf{default behavior}. 
The default behavior is introduced and demonstrated in Section~\ref{sec:shape persistence}, and the theory supporting the algorithm is described in Section~\ref{sec:shape_persistence_theory}.
The later portions of the paper are dedicated to the algorithms responsible for shape adaptability which are only executed to resolve a shape change.
The adaptability algorithms are introduced in Section~\ref{sec:shape_adaptability}.
In Section~\ref{sec:additional_experiments}, we provide demonstrations of both adaptability and persistence in the presence of humans and for a large swarm.
Finally, in Section~\ref{sec:adaptability_theory} we describe the theory supporting the adaptability algorithms, and in Section~\ref{sec:conclusion} we summarize our work and describe our future plans.
Terms defined in the text are indicated in \textbf{bold} and their definitions are included in Appendix~\ref{app:terms}.

\section{Related Work}\label{sec:related work}
This work lies in a largely unstudied area of swarm robotics with no known direct state-of-the-art comparisons.
However, related works exist in the fields of shape formation, area coverage, Hamiltonian path planning, and human-swarm interaction. 

\subsection{Shape Formation}
One approach to persist in a shape is to have individual swarm robots cycle into and out of the shape to recharge.
Here, the energy limitations of individual robots no longer constrain shape duration but instead constrain the size of the path that an individual robot can travel and thus the size of the final shape.
%In that sense, a robot's ability to travel along its path and return to a charging station can become a limiting factor on the swarm's ability to produce shapes. This is especially true if
One problem occurs if a robot is ever ``stuck" along its path in the middle of the shape formation, unable to return to a charging station because it is surrounded by other robots forming the shape.
Stuck robots are a potential problem for most static shape formation algorithms that begin with robots in some arbitrary starting position and end with robots statically holding a position in the defined shape \cite{wang2020shape,xu2010multi-robotpattern, xue2009formationcontrol,wu2021distributed,liu2018distributedcontrol,aranda2016distributed}.

However, there are some non-static shape formation algorithms that use dynamic shape scaling \cite{wang2021generating}, time varying shape formations \cite{dong2015time-varying}, and robot position swapping \cite{xu2019concurrent} so that robots are not always statically holding a position in the defined shape. 
These could potentially be altered to avoid stuck robots, but they either lack guarantees that a particular robot will never become stuck or they operate on predefined shapes and therefore lack adaptability to real-time changes.

\subsection{Area Coverage}
Traditional area coverage algorithms are similar to adaptive and persistent shape formation because they can cycle robots through a defined shape or area, though typically at a lower density of robots than most shape formation algorithms.
Dozens of area coverage algorithms already exist for aerial vehicles \cite{cabreira2019survey}. 
Many of these algorithms use traditional back-and-forth path planning with a centralized controller \cite{jiao2010research, huang2001optimal,torres2016coverage, difranco2016coverage}.
Others use a centralized controller to create spiral-like paths that end in the middle of the area, making it challenging for a robot to return to a charging station for persistent coverage without becoming stuck~\cite{cabreira2018energy-aware}. 
Still others implement decentralized approaches (which are more suitable for swarms) using built-in randomness \cite{albani2017field}, cost functions \cite{lim2010waypoint}, and virtual pheromones \cite{zelenka2014insect-losscomms}.
However, the decentralized approaches typically result in nonuniform area coverage which, if used for shape formation, would result in distorted shapes with no guarantees for avoiding stuck robots. % that an individual could exit the shape to recharge.

\subsection{Hamiltonian Path Planning}
In this work, we ensure that a robot can cycle through a 2D shape and back to a charging station by using planar (i.e., non-intersecting) Hamiltonian cycles.
These guarantee the robot's path covers the entire shape uniformly.
If the cycle is further constrained to have adjacent start and end positions on the periphery of the shape, then the robot is guaranteed to be able to cycle into and out of the shape without becoming stuck or colliding with its peers.
% In that case, the Hamiltonian path could be considered a Hamiltonian cycle if its start and end nodes were connected by a path edge.
Unfortunately, finding a Hamiltonian cycle through a shape represented as a grid graph of points in space is a NP-complete problem \cite{Garey1976planar,umans1997hamiltonian, nishat2020reconfiguration}, making it difficult to scale to large shapes. 

Despite the challenge, others have found Hamiltonian cycles by either restricting the problem to grid graphs that are solid (i.e., no holes)~\cite{umans1997hamiltonian} or for specialized purposes such as programmable chain assembly~\cite{griffith2011programmable}, obstacle avoidance~\cite{joshi2019modification}, and nonuniform aerial coverage \cite{sadat2015fractal}.
However, in all of these approaches, Hamiltonian cycles are found via a centralized planner using global information about the entire grid graph or a single robot moving through (and sensing) the entire environment. 
In a swarm setting, it is important that Hamiltonian cycles are generated in an online, decentralized manner with each individual robot planning its next step using only local information.
This helps the swarm stay adaptable and permits the use of simple robots.

\subsection{Human-Swarm Interaction}
Kolling et al. \cite{kolling2015human} define humans directly interacting with a swarm in a shared environment as proximal interaction.
Examples of proximal interaction include works where swarms are shown to detect and identify specific human gestures \cite{giusti2012human, nagi2014human}.
In our work, we use human gestures in a shared environment to initiate shape changes.
We focus on the swarm's response to the gesture, not necessarily the detection and identification of the gesture itself.

\section{Robot Capabilities and Assumptions}\label{sec:robot capabilities}
The algorithms presented in this paper are agnostic to specific swarm implementations, so the algorithms can be employed on robots with various hardware and software designs. 
However, there are some basic capabilities we assume each robot has in order to execute the algorithms.

First, we assume that each robot can store a representation of a grid graph overlay on the environment ($G = (n,e)$).
We also assume that each robot can identify if it is residing at a particular node ($n$) or translating along a particular edge ($e$) of $G$.
Further, we assume that each robot can identify a subset of grid graph nodes that represent the shape ($S$) that the swarm is forming in space ($S \subset G$).
For the simulations and physical demonstrations in this paper, we elected (by design choice) to have robots maintain a list of in-shape grid nodes.
If the shape is ever changed, robots communicate the change using neighbor-to-neighbor communication to propagate the change through the swarm so each robot can update its list.

For this to occur, we also assume that robots are capable of robot-to-robot communication and are equipped with internal clocks.
We also assume that each robot can communicate with its peers within a range of $\sqrt{2}*l$ where $l$ is the length of an edge in $G$. 
This allows robots to communicate with up to eight immediate neighbors (i.e., N, NE, E, SE, S, SW, W, and NW). 
During operation of the algorithms, robots use neighbor-to-neighbor communication to inform one another of shape changes, share position information, and synchronize their clocks so that robots can move together in lockstep.

Additionally, since the algorithms in this paper use planar Hamiltonian cycles (and since finding a Hamiltonian cycle in a grid graph is an NP-complete problem), it is not realistic to assume that a given robot can find a Hamiltonian cycle as it navigates through a very large shape.
Therefore, we enforce a set of rules to standardize shape structures so that the swarm can more easily find Hamiltonian cycles, including for very large shapes and shapes with holes.

The rules define a unit shape, called a \textbf{box}, that is comprised of four grid nodes in a 2x2 square (Fig.~\ref{fig:terms}).
This is similar to the method of finding Hamiltonian paths by breaking a box down into four smaller boxes, as discussed for space filling curves in \cite{griffith2011programmable}.
We assume that each robot associates a given 2x2 square of grid nodes with the same box as every other robot.
This ensures that a given shape representation is consistent for all robots in the swarm.
Further, we assume each robot can determine the box that it is in and its relative position with respect to the center of that box. 
This allows robots to detect if they are moving clockwise or counter-clockwise around the center of a box.
% It also allows robots to detect which quadrant they are in with respect to the center of the box.

Allowable shapes are continuous so that each box meets the full side of another box: partial side connections and vertex-to-vertex only connections are not allowed.
Partial boxes are also not permitted.
Shapes that meet these conditions are called \textbf{valid shapes}, and they are important because they reduce the difficulty of finding a planar Hamiltonian cycle from NP-complete to something that is solvable in linear time (run time, $O(\mathcal{B})$, scales linearly with $\mathcal{B}$, where $\mathcal{B}$ is the number of boxes in the shape). 
In fact, we will show that a planar Hamiltonian cycle can always be found for any valid shape in Section~\ref{sec:valid_shapes}.
Finally, although we only consider valid shapes created in a square lattice graph, it is plausible that the algorithms in this paper could extend to other planar graphs, such as hexagonal lattices, but this is reserved for future work.

\begin{figure}[t]
\centering
\includegraphics[width=80mm]{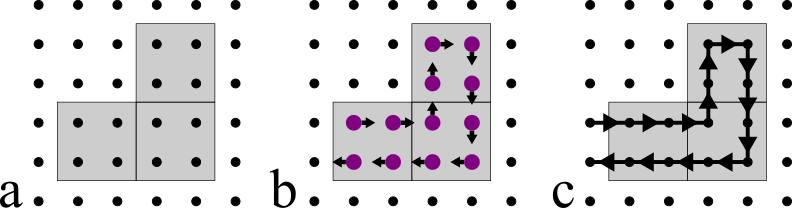}
\caption{(a) An example of a valid shape of 3 boxes. Points indicate nodes of $G$. (b) Robots (purple) approximating shape. Arrows indicate robot heading. (c) The path along which robots travel. Arrows indicate edge direction.
}
\label{fig:terms}
\end{figure}

\section{Shape Persistence}\label{sec:shape persistence}
For a swarm to form a persistent shape beyond the energy limitations of any individual robot, the swarm must cycle robots back and forth between the shape and a charging station. 
This cycle can be separated into four subroutines: 1) forming the shape, 2) traveling from the exit point of the shape to a charging station, 3) charging at the charging station, and 4) traveling from the charging station to the entry point of the shape. 
The primary contributions of this work are in the first subroutine: forming the shape. 
To do this, each robot will travel along the same planar Hamiltonian cycle to both approximate the shape and move through the shape without colliding with other robots in the swarm.

To be consistent with prior works, we have adapted the formal definition of Hamiltonian circuits from the 1976 paper by Garey et. al. \cite{Garey1976planar} which first identified the planar Hamiltonian circuit problem as NP complete (Definition~\ref{def:planar_hamiltonian_cycle}).
We will use this definition throughout the remainder of this paper.

\begin{definition}[Planar Hamiltonian Cycle]\label{def:planar_hamiltonian_cycle}
A planar Hamiltonian cycle in a graph is a path which passes through every vertex exactly once and returns to its starting point without intersecting itself.
\end{definition}

Further, in our work, planar Hamiltonian cycles are broken at the adjacent entry and exit points (there are no path edges connecting the two). 
Thus, the paths that robots create through shapes are technically planar Hamiltonian paths (not cycles). %with adjacent start and end nodes. 
However, for the remainder of this paper, a planar Hamiltonian path with adjacent start and end nodes will be treated identically to a planar Hamiltonian cycle since the former can always be turned into the latter by simply connecting the adjacent start and end nodes. 
Both will be referred to as Hamiltonian cycles to distinguish from other Hamiltonian paths that do not have adjacent start and end nodes.
Finally, we restrict the start and end nodes to the periphery of the shape so that each robot can enter and exit the shape and ultimately cycle to and from a charging station.

\subsection{Subroutine 1: Forming the Shape}
% By definition, a valid shape is a subset of the digital grid graph representation overlaid on the environment. 
% Nodes within that subset that are along the perimeter of the shape are known as \textbf{periphery nodes}. 
% One of the periphery nodes represents the entry node where robots coming from the charging station can enter the shape.
% An adjacent periphery node represents the exit node where robots can leave the shape to return to the charging station.
% subroutine 1 is focused on all robot behavior between these entry and exit nodes.

In subroutine 1, robots form a planar Hamiltonian cycle to both approximate the shape and cycle robots from the entry node (a grid node on the periphery of the shape where robots enter the shape) to the exit node (a grid node on the periphery of the shape where robots exit the shape to return to the charging station). 
To do this, each robot enters the shape at the entry node and then visits each grid node in shape exactly once before finally exiting the shape at the exit node without ever intersecting its own path.
Robots keep track of the boxes they have visited $V_B$ and the nodes they have visited $V_n$ while traversing through the shape.
Robots follow a set of rules called the \textbf{default behavior} when traveling in the shape to determine the order in which grid nodes are visited between the entry and exit nodes using only local information.
% The default behavior rules ensure that the decisions each robot makes to travel from one node to the next will lead the robot through the entire shape via the same path as every other robot in the swarm.

%Point: Default behavior are the rules that make preferred methods consistent.
The default behavior prioritizes a robot's movement through a shape so that a robot visits new boxes in a manner that mimics a traditional depth-first search. 
If a robot were performing a depth-first search of the boxes in a shape, the robot would first identify all of the children of its current box.
Next, the robot would enter a child box and continue to move parent-to-child as it explored all of the descendant boxes within that subtree.
Only after the robot exhausted all descendants in the subtree would the robot move ``up" the tree, revisiting boxes in the reverse order (child-to-parent) and repeating the process of visiting descendant boxes before visiting other children.
Once all boxes had been visited, the robot would exit the shape.

Similar logic governs the default behavior.
As a robot traverses through the shape, it moves from node-to-node such that it visits boxes in a depth-first, clockwise priority manner starting at the entry node and finishing at the exit node.
Specifically, when a robot is at a node in the shape, there is a set of four possible edges that the robot could take to move away from that node (i.e., $e_N$, $e_S$, $e_E$, $e_W$). 
If the robot is at the exit node, it will take the edge that leads the robot out of the shape to the charging station.
For all other nodes ($n$) in the shape ($S$), the robot selects which edge to take next by first ignoring any edge that leads the robot to a node that is out of the shape or to a node that the robot has already visited. 
Then, the following rules are applied to the remaining edges in the set (as depicted in Algorithm~\ref{alg:default_behavior}):
% \begin{enumerate}
%      \item \textbf{IF} an edge leads to a node in a box not previously visited, take that edge. (Rule 1)
%      \item \textbf{ELSE IF} an edge leads the robot clockwise within its current box, take that edge. (Rule 2) 
%      \item \textbf{ELSE} an edge leads to a node in the parent box in the tree, take that edge. (Rule 3)
%  \end{enumerate}
\begin{quote}

\textbf{IF} an edge leads to a node in a box not previously visited, take that edge. (Rule 1)

\textbf{ELSE IF} an edge leads the robot clockwise within its current box, take that edge. (Rule 2) 

\textbf{ELSE IF} an edge leads to a node in the parent box in the tree, take that edge. (Rule 3)
    
\end{quote}

\begin{algorithm}[ht]
\caption{Default Behavior}
\label{alg:default_behavior}
    \begin{algorithmic}[1]
        \Require $b, S, V_B, V_n$ \Comment{$b$ = current box}
        \Ensure $\epsilon'$ \Comment{$\epsilon'$ = edge to next node}
        \If{robot at exit node}
            \State exit shape
        \Else
            \Statex \qquad$\triangleright$ \textbf{Create lists of plausible next edges and nodes}
            \State $\mathcal{E} \gets [e_N, e_S, e_E, e_W]$
            \State $\mathcal{N} \gets [n_N, n_S, n_E, n_W]$
            \Statex \qquad$\triangleright$ \textbf{Remove visited and out-of-shape options}
            \For{$\epsilon, n$ in $\mathcal{E}, \mathcal{N}$}
                \If{$n \in V_n$}%in visited nodes
                    \State remove $\epsilon, n$ from $\mathcal{E}, \mathcal{N}$
                \ElsIf{$n \notin S$}%not in shape
                    \State remove $\epsilon, n$ from $\mathcal{E}, \mathcal{N}$
                \EndIf 
            \EndFor 
            \Statex \qquad$\triangleright$ \textbf{Employ default behavior rules}
            % \For{$\epsilon, n$ in $edges, nodes$}
            %     \If{Box($n$) $\notin V_B$}\Comment{Rule 1}%in unvisited box}
            %         \State set $\epsilon'$ to $\epsilon$ then \textbf{break}\Comment{Rule 1}
            %         % \State $\epsilon' \gets \epsilon$
            %         % \State break
            %     \ElsIf{$n$ in $b$ \textbf{and} $\epsilon$ is clockwise}\Comment{Rule 2}
            %         \State set $\epsilon'$ to $\epsilon$ then \textbf{break}\Comment{Rule 2}
            %         % \State $\epsilon' \gets \epsilon$
            %         % \State break
            %     \ElsIf{$n$ in parent of $b$}\Comment{Rule 3}
            %         \State set $\epsilon'$ to $\epsilon$ then \textbf{break}\Comment{Rule 3}
            %         % \State $\epsilon' \gets \epsilon$
            %         % \State break
            %     \EndIf
            % \EndFor
            \If{$\exists$ $\epsilon,n \in \mathcal{E},\mathcal{N}$ $|$ Box($n$) $\notin V_B$} \Comment{Rule 1}
                \State $\epsilon' \gets \epsilon$ \Comment{Rule 1}
            \ElsIf{$\exists$ $\epsilon,n \in \mathcal{E},\mathcal{N}$ $|$ $n$ in $b$\Comment{Rule 2}\\ \qquad\qquad\qquad\quad \textbf{and} $\epsilon$ is clockwise} \Comment{Rule 2}
                \State $\epsilon' \gets \epsilon$ \Comment{Rule 2}
            \ElsIf{$\exists$ $\epsilon,n \in \mathcal{E},\mathcal{N}$ $|$ $n$ in parent of $b$} \Comment{Rule 3}
                \State $\epsilon' \gets \epsilon$ \Comment{Rule 3}
            \EndIf
        \EndIf 
    \end{algorithmic}
\end{algorithm}

% \If{$n_\epsilon = w_0 \land n_R = n_e$}
%             \State $\epsilon' \gets \epsilon$
%             \State \Return $\epsilon'$
%         \Else
%             \State $\lambda[i] \gets 0$
%         \EndIf

% The default behavior is described in pseudo-code in Algorithm~\ref{alg1}. 
% It is assumed that a robot has sensed its current node ($n_R$), box ($b_R$), and its quadrant within that box ($q_R$).
% It is also assumed that the robot knows the exit node of the shape ($n_e$) and the first waypoint on the way back to the charging station ($w_0$). 
% As the robot executes its default behavior, it maintains a record of the boxes it has visited ($\beta$), the nodes it has visited ($V$), and the nodes that are known to be within the shape ($S$). 
% Every time a robot needs to identify its next move, it uses Algorithm~\ref{alg1} to determine the next edge it should take ($\epsilon'$) in accordance with the default behavior.

While executing the default behavior, robots communicate timing information to synchronize their clocks.
This allows robots to move in lockstep with one another, avoiding in-path collisions and maintaining spacing such that one robot is always occupying one grid node in $S$. 
Robots also avoid collisions by never crossing another robot's path.
This is the result of every robot following the same default behavior rules (and Hamiltonian cycle).
Since the default behavior creates a Hamiltonian cycle such that each robot terminates its path at the exit node (which is on the periphery of the shape), and every robot can traverse the shape without collision, we can guarantee that all robots will be able to exit the shape to recharge (i.e., no “stuck” robots), which facilitates the persistence of the shape.
The fact that the default behavior results in a planar Hamiltonian cycle is proven in Section~\ref{sec:default_makes_preferred}.

One key aspect of the default behavior is that robots do not plan their paths through the shape prior to entering it.
Instead, each robot makes its own decisions to reactively build a path through the shape from entry to exit one step at a time.
That means that our design choice to have robots maintain a list of in-shape grid nodes is not necessary for algorithm operation.
If a robot were capable of distinguishing between in-shape and out-of-shape nodes with its onboard sensors, then the list of in-shape grid nodes ($S$) could be replaced with a list that the robot senses and creates as it moves.
This more sophisticated sensing scheme would replace the shape memory requirements in the default behavior without changing the rules for robot motion.
This would further expand the practical applications of persistent shape formation to very large shapes (with many grid nodes) or for very simple robots (with limited memory capacity).
Investigating this is reserved for future work.

\subsection{Subroutines 2, 3, \& 4}
Although the remaining subroutines are not the primary focus of this work, they are necessary for the swarm's success, so very simplistic behaviors are implemented.
The two ``traveling" subroutines (traveling to the charging station and traveling to the shape) are simplified by a set of predefined, static waypoints from the exit node to the charging station and from the charging station to the entry node, respectively. 
This ensures that robots can travel between the charging station and the shape without encountering obstacles or collisions. 

The traveling subroutines are further simplified such that robots arrive at the shape and exit the shape at the same user defined constant time interval as robots moving in lockstep between nodes within the shape (e.g., $\tau = 10$s). 
Further, if the robots in the shape are ever holding still to complete some portion of an algorithm, then the last robot to have left the shape will carry a message to the other robots outside of the shape.
This message contains a list of nodes added to or removed from the shape so that the robots outside of the shape can update their list of in-shape nodes.
Robots in receipt of this message also adjust their shape-arrival times to limit the number of robots queuing at the shape entry node.

Similarly, the entry and exit nodes are simplified to be static and known $a$ $priori$.
By definition, the entry and exit nodes are adjacent to one another on the periphery of the shape and within the same box.
% Additionally, the entry and exit nodes were selected to be compatible with a unit path (i.e., a non-spanning edge drawn from exit node to entry node would result in clockwise motion around the center of the entry and exit node box).
Having static and predefined entry and exit nodes allows us to demonstrate the shape persistence aspects of the algorithms without unnecessarily complicating swarm behavior.
Changing the entry and exit nodes is reserved for future work.
 
The charging subroutine is also simplified since most implementations would be highly dependent on the type of charging technology used (e.g., wireless charging vs. contact charging).
For this work, the charging station is modeled as a set of static positions in space. 
Upon reaching a position in the charging station, a robot begins charging. 
When a robot is done charging, it leaves the charging station to travel back to the shape in a first-in-first-out order and such that it will arrive at the shape at the constant time interval ($\tau$).

In addition to these simplifications, we assume that robots move through the entire shape and back to the charging station on one charge and that the swarm has enough robots that each robot has enough time to completely recharge before it cycles back into the shape. 
If for a given implementation a shape is too large that robots could not pass through the shape and back to the charging station on a single charge, a system of multiple charging stations could be implemented, effectively breaking the shape into multiple pieces, but that is reserved for future work.

\subsection{Demonstrations of Persistence}
Shape persistence was demonstrated in identical tests for both a simulated swarm and a swarm of physical robots.
In both kinds of experiments, a user initialized a shape of four boxes (with a start and end node) in a grid graph with edge lengths of $l = 0.2$m.
16 robots were set up in a temporary queue outside of the shape, and 22 robots were placed in a charging station.
Coordinates of the charging station, as well as the waypoints to and from the charging station, were pre-programmed and provided to all robots.
These coordinates did not change during the course of the experiment.
Communication was artificially limited to a range of 0.3m for each robot.

After initialization, the robots in the temporary queue began to fill the shape via the default behavior. 
Each robot stepped in synchronization with other robots (due to clock synchronization messages), moving to a new node every $\tau = 12$s. 
Once the shape was almost completely formed, robots in the charging queue began to leave so that they would arrive at the shape in $\tau = 12$s increments starting as soon as the first robot was about to exit the shape.
For every 12s after that, one robot left for the charging station, one robot arrived from the charging station, and each robot within the shape moved to its next node.
At that point, the system was in steady state and allowed to persist for some time.

In both the simulation and physical experiments, the charging station was modeled as a set of linear positions along one side of the environment. 
This was meant to emulate a ``charging rail" where robots could pull up to charge themselves. 
This design decision did create the potential for collisions between robots entering a charging station position and robots exiting a charging station position. 
Robots communicated their own position to one another to help avoid collisions.
In the event of a potential collision, robots heading to the charging station yielded to robots heading toward the shape to ensure the shape-bound robots arrived at the shape on time.
Furthermore, robots entered and exited the charging station in a first-in-first-out basis.

For the physical robot demonstration, we used an existing swarm of ground robots called \emph{Coachbot V2.0}~\cite{wang2020shape}. 
Each \emph{Coachbot} is equipped with a two-wheeled differential drive system and the ability to sense its position $(x,y,\theta)$ in the arena. 
The robots are also equipped with an onboard battery, an onboard Raspberry Pi computer, and onboard Wi-Fi for robot-to-robot communication that can be artificially limited to a user defined distance.
The simulator was built to emulate the performance of the \emph{Coachbots} (i.e., differential drive robots with the ability to sense $x$, $y$, and $\theta$, robot-to-robot communication, etc.).

\begin{figure}[!t]
\centering
\includegraphics[width=80mm]{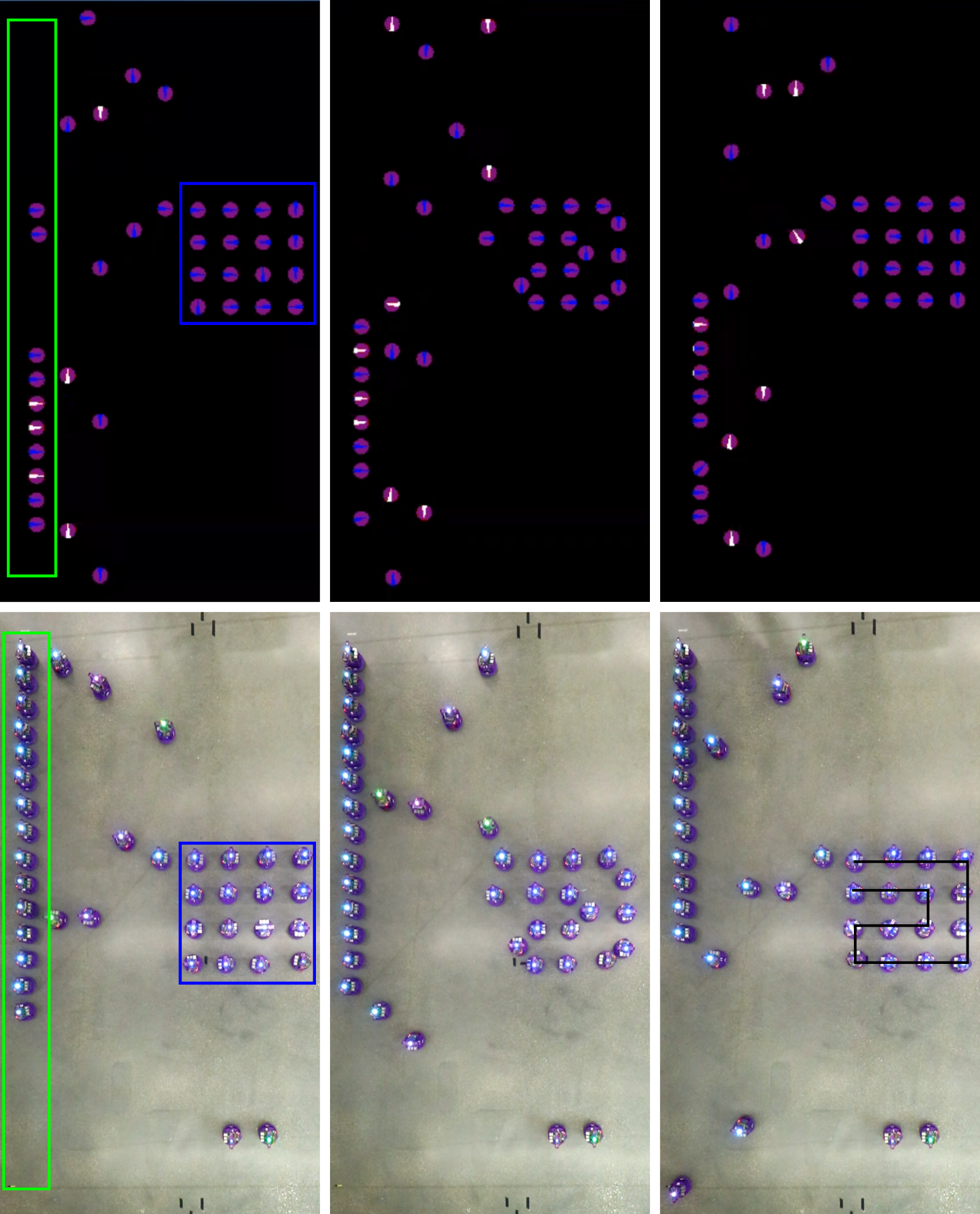}
\caption{Images from a simulation (top) and physical robot experiment (bottom). 
Each set of three images represents a sequence in time of the robots forming the shape (left), stepping to their next node (center), and holding at their next node (right).
The charging station is highlighted by the green box, and the desired shape is highlighted by the blue box.
Robots in neither box are cycling to or from the shape.
The resulting path through the shape for both the simulation and the physical robots is shown in the bottom right image.
The LED color of each robot is insignificant for these experiments.
}
\label{persistance}
\end{figure}

Images from both the simulation and physical robot experiments can be found in Fig.~\ref{persistance}.
% The results from both types of tests suggest that the default behavior is an effective means of persistent shape formation.
Both the simulated swarm and the physical swarm were left to maintain their respective shape for over 30 minutes (and some simulations were left to run for well over an hour).
In that time, each robot cycled through the shape and back to the charging station 4 or 5 times (9 or 10 times for the hour-long simulations).
The experiments were stopped by an experimenter, not due to robot or algorithmic failure. 
This suggests that each swarm is capable of persistent shape formation for extended duration tasks because both the simulated swarm and the physical swarm demonstrated the ability to cycle robots through the shape via the default behavior without failure or collision for multiple cycles through the entire swarm of robots.
% Although it is impossible to demonstrate this behavior working as time approaches infinity, the results demonstrate each swarm's ability to cycle robots through a shape using the default behavior.

The \emph{Coachbots} were used for the physical demonstrations because they are an established swarm platform that was available to the authors.
However, as ground robots, they could simply drive to form a shape, power down, and persist in that shape indefinitely without needing to recharge.
In the future, we plan to demonstrate swarm persistence with a swarm of flying robots (currently under construction) that do not have the luxury of a floor to help them hold position in space.
In the meantime, the \emph{Coachbots} served as an experiment-ready option that enabled us to demonstrate the algorithms on a physical platform.

\section{Shape Persistence Theory}\label{sec:shape_persistence_theory}
%The previous sections described how a robot can use the default behavior to generate a planar Hamiltonian cycle to both approximate the shape and move through the shape, ultimately returning to a charging station.
The theory behind the default behavior is rooted in some provable properties of valid shapes and planar Hamiltonian cycles.
The remainder of this section formally addresses this theory in three parts: establishing basic concepts, proving planar Hamiltonian cycles can be found for any valid shape, and proving the default behavior always produces a planar Hamiltonian cycle.

\subsection{Basic Concepts}\label{sec:basic_concepts}
In order to simplify the discussion of shape persistence theory, we define some basic terminology and concepts. 
While most of these concepts are used in the shape persistence proofs, some of them are also applicable to the shape adaptability concepts discussed in later sections.

The first concept is the \textbf{unit shape}: a shape of $b=1$ box.
A robot can take one of two possible cycles through a unit shape:  clockwise or counter-clockwise (Fig.~\ref{fig:unit_path}).
By design choice, we elect to use the clockwise cycle as default in this paper.
We define a clockwise planar Hamiltonian cycle through a unit shape as the \textbf{unit path}.
Choosing the counter-clockwise cycle as the unit path would have also been valid.
In that case, the algorithms would result in robots moving along paths (and edges) in the opposite directions as the ones shown in this paper.
Furthermore, Algorithm~\ref{alg:default_behavior} rule 2 would have to be changed to prioritize counter-clockwise motion instead of clockwise.
The remainder of the concepts and theories presented in this paper would go unchanged.

\begin{figure}[!t]
\centering
\includegraphics[width=40mm]{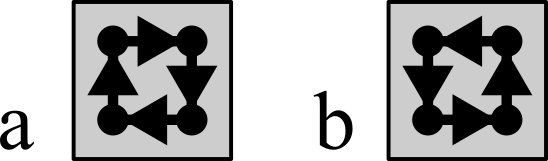}
\caption{(a) A clockwise planar Hamiltonian cycle through a unit shape. (b) A counter-clockwise planar Hamiltonian cycle through a unit shape. Arrows indicate edge direction.
}
\label{fig:unit_path}
\end{figure}

In discussing the remaining concepts, we will use the following terminology: \textbf{parallel} to refer to path edges that are equidistant everywhere and do not intersect; \textbf{opposite} to refer to edges that have directions exactly 180\textdegree{} from one another;
\textbf{spanning} to refer to path edges that begin and terminate in different boxes; and \textbf{non-spanning} to refer to path edges that begin and terminate in the same box.
Finally, we will define two path edges as a \textbf{pair} if the start and end nodes of both path edges are in a 2x2 grid configuration (regardless of if that grid configuration is within a box or across multiple boxes).
Thus, the unit path is comprised of two pairs of parallel and opposite non-spanning edges (a pair directed North-South and a pair directed East-West).

With these terms, we can discuss merging two paths into one or separating a single path into multiple and,
in the process, establish two lemmas.

\begin{lemma}\label{lem:merging}
    If path merging is performed on two separate planar Hamiltonian cycles, then the result will be one combined planar Hamiltonian cycle.
\end{lemma}

\begin{lemma}\label{lem:separation}
    If path separation is performed on a planar Hamiltonian cycle, then the result will be separate planar Hamiltonian cycles on either side of the box connection where the separation occurred.
\end{lemma}

The first concept, \textbf{path merging}, involves separate paths merging into one path by creating a pair of spanning edges across a box connection.
During path merging, a pair of parallel and opposite non-spanning edges (one in each box on either side of the box connection) are replaced with a pair of parallel and opposite spanning edges (Fig.~\ref{fig:operations}a to b). 
The resulting path is a planar Hamiltonian cycle through all of the nodes in the combined path by Definition~\ref{def:planar_hamiltonian_cycle}.
Further, the only edges affected by this change are the pairs of edges that go from non-spanning to spanning; the rest of the edges and nodes are unchanged.
Therefore, it does not matter if the pre-merge paths were simple planar Hamiltonian cycles (such as the unit paths shown in Fig.~\ref{fig:operations}a), or if the pre-merged paths were larger, more complicated planar Hamiltonian cycles. 
Because the changed portions of the path are within Definition~\ref{def:planar_hamiltonian_cycle}, and since all other portions are unaffected, the result is still a planar Hamiltonian cycle. 
Thus, we can conclude Lemma~\ref{lem:merging}.

\begin{figure}[!t]
\centering
\includegraphics[width=50mm]{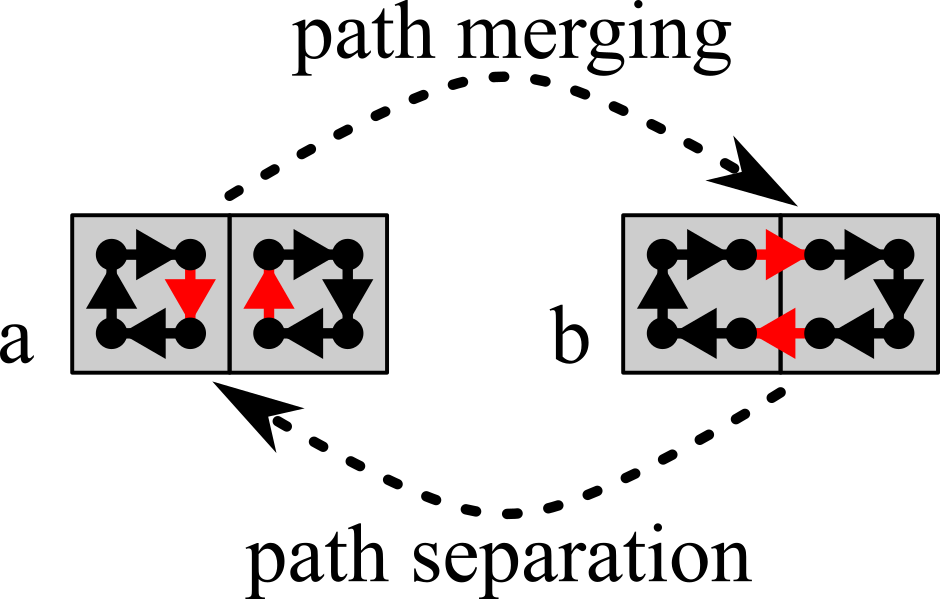}
\caption{The fundamental operations of path merging (going from $a$ to $b$) and path separation (going from $b$ to $a$). Red indicates impacted edges. Arrows indicate edge direction.
}
\label{fig:operations}
\end{figure}

The opposite concept, called \textbf{path separation}, involves one planar Hamiltonian cycle separating into multiple paths.
For this to occur, a pair of parallel and opposite spanning edges are replaced with a pair of parallel and opposite non-spanning edges (Fig.~\ref{fig:operations}b to a).
The result is two planar Hamiltonian cycles on either side of the box connection after separation.
Similar to path merging, the only edges affected are the pair of edges that go from spanning to non-spanning; the remainder of the shape nodes and path edges from the pre-separated planar Hamiltonian cycle are unchanged.
Therefore, it does not matter if the post-separated paths are simple unit paths, as shown in Fig.~\ref{fig:operations}a, or if the post-separated paths are larger, more complicated planar Hamiltonian cycles. 
In either case, path separation will always result in planar Hamiltonian cycles on either side of the box connection after the separation.
Thus, we can conclude Lemma~\ref{lem:separation}.

\subsection{Planar Hamiltonian Cycles in Valid Shapes}\label{sec:valid_shapes}
Earlier, we claimed that a planar Hamiltonian cycle can always be found for any valid shape, and we suggested that the process of finding a cycle would scale linearly with the number of boxes in the shape ($O(\mathcal{B})$ scales linearly with $\mathcal{B}$).
To prove these claims, we establish four things: the assembly of valid shapes by sequential box addition, path merging during shape assembly, \textbf{periphery nodes}, and \textbf{periphery edges}. 

First, consider the assembly of valid shapes. 
Since valid shapes are constructed from boxes, any valid shape can be assembled by adding boxes together in an arbitrary sequence until the desired final shape is achieved.
For example, the shape in Fig.~\ref{fig:random_box_add}f could be constructed by assembling boxes in the order 0,4,3,2,1,5 or 0,1,2,3,4,5, etc.
We can further constrain assembly so boxes are only added together such that a valid shape is maintained after each addition (i.e., each additional box is added to be adjacent to an existing box).
To continue the example under such restrictions, 0,4,3,2,1,5 would no longer be a valid sequence for assembling the shape in Fig.~\ref{fig:random_box_add}f, but 0,1,2,3,4,5 would be valid because it maintains a valid shape after each additional box.
% Maintaining a valid shape after each additional box will allow us to examine planar Hamiltonian cycles at each stage of the assembly.

Next, consider the concept of path merging during shape assembly under the valid shape constraint. 
Every time a box is added to a valid shape, its unit path could be merged with the existing path via path merging. 
By Lemma~\ref{lem:merging}, this would result in a new planar Hamiltonian cycle through the shape.
For example, when box 2 is assembled with boxes 0 and 1 in Fig.~\ref{fig:random_box_add}c, the unit path in box 2 is merged with the existing path in boxes 0 and 1. 
Specifically, the bottom edge of the path in box 0 and the top edge of the unit path in box 2 are the pair of parallel and opposite non-spanning edges that are turned into spanning edges to merge the paths.
The result is a new planar Hamiltonian cycle through boxes 0, 1, and 2.

\begin{figure}[t]
\centering
\includegraphics[width=80mm]{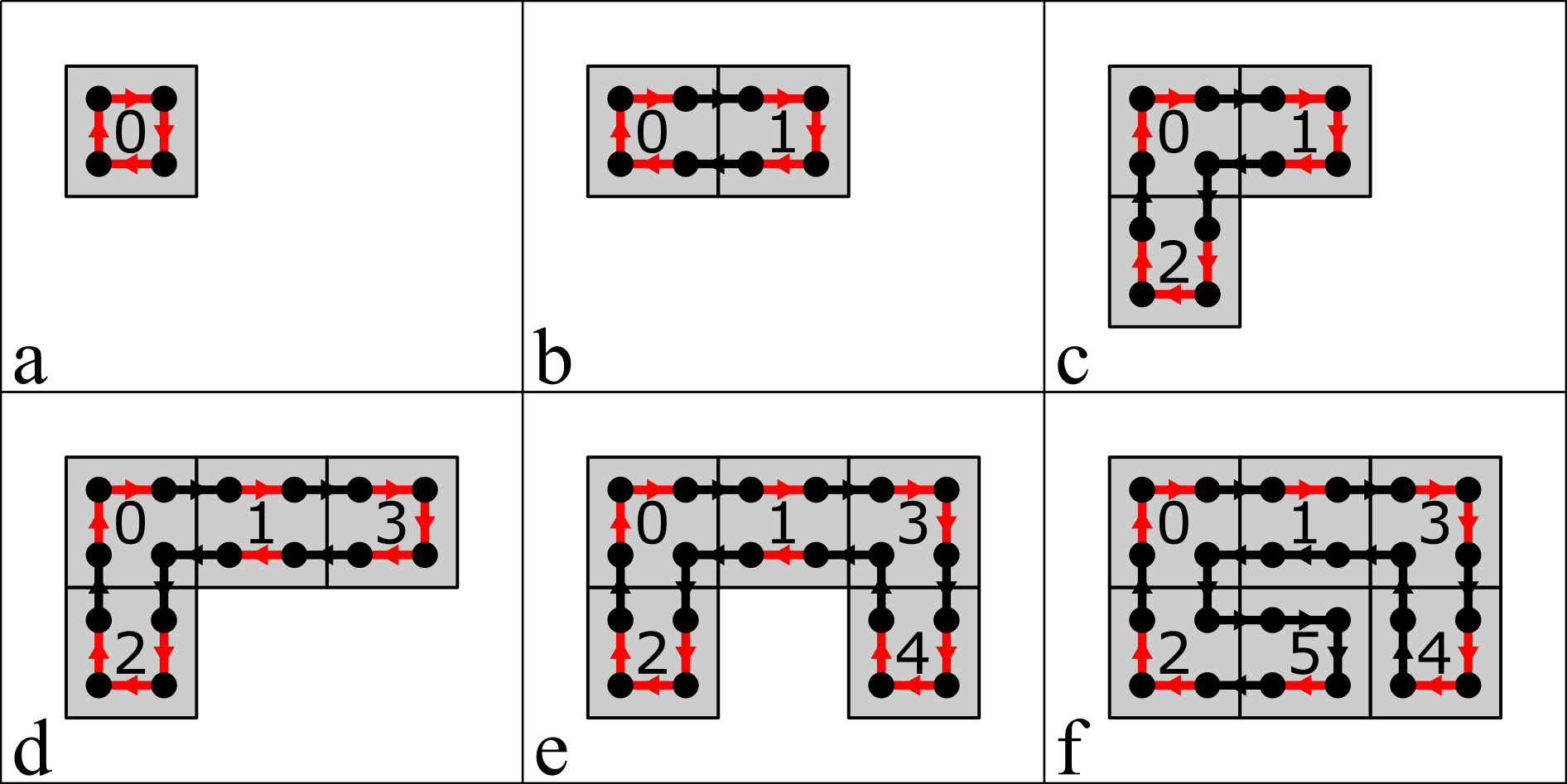}
\caption{An arbitrary sequence of shape construction (a-f). Each additional box creates a new valid shape and valid path. Periphery edges are red. Arrows indicate edge direction.
}
\label{fig:random_box_add}
\end{figure}

Finally, we define a \textbf{periphery node} as an in-shape node on the periphery of the shape and a \textbf{periphery edge} as a non-spanning edge between two periphery nodes.
Periphery edges represent edges that could be merged with unit paths in new boxes when new boxes are added to the shape.
If every periphery node is connected by at least one periphery edge, then a new box could be added anywhere along the periphery of the shape and its unit path could always merge with the existing path.
For example, every periphery node in the shape in Fig.~\ref{fig:random_box_add}c is connected by at least one periphery edge, so a new box could be added to any side on the periphery of the three-box shape.

%With these topics introduced, we can now prove our claim.

\begin{theorem}\label{claim_random}
A planar Hamiltonian cycle exists for any valid shape.
\end{theorem}

\begin{proof}[\proofname\ 1]\label{prf:hamiltonian_cycles}
This is a proof by induction. 
Consider the base case of the unit box ($\mathcal{B}=1$).
This shape has three important properties. 
Trivially, it is a valid shape (property 1) and its path is a planar Hamiltonian cycle (property 2).
Further, all periphery nodes are connected by at least one periphery edge (property 3).
For example, see Fig.~\ref{fig:random_box_add}a.

Next, assume there exists a shape of $\mathcal{B}=K$ boxes, where $K$ is some integer and $K\geq1$. 
Assume the shape has properties 1, 2, and 3. 
A new valid shape of $\mathcal{B} = K+1$ boxes could be generated by adding a new box next to an existing box in the $\mathcal{B} = K$ shape such that the two boxes share an entire side. 
This maintains property 1.
The unit path in the newly added box could be merged with one of the periphery edges on the existing path via path merging to maintain property 2 (by Lemma~\ref{lem:merging}).

Further, by creating the new path in this way, property 3 is maintained.
This is because the portions of the existing path not involved in the path merge are unchanged and any of the sides of the new box that are on the periphery of the shape have a periphery edge between two periphery nodes.
Neither of the edges used in the path merge are on the periphery of the shape, so the $\mathcal{B} = K+1$ path maintains property 3. 
Examples of this can be seen in each box addition in Fig.~\ref{fig:random_box_add}.

Since all three properties are maintained for the case of $\mathcal{B}=1$ and for the transition from $\mathcal{B}=K$ to $\mathcal{B}=K+1$ for some $K\geq1$, we can conclude by induction that Theorem~\ref{claim_random} is valid for any valid shape of size $\mathcal{B}\geq1$.
\end{proof}

Based on Proof~\ref{claim_random}, since we can find a Hamiltonian cycle through a shape by incrementally adding boxes together until we reach that shape, we can conclude that time to find the Hamiltonian cycle would scale with the number of boxes in the shape ($O(\mathcal{B})$ scales linearly with $\mathcal{B}$).

Proof~\ref{claim_random} does not rely on the order in which boxes are assembled to form a shape so long as each additional box results in a valid shape.
Furthermore, when a new box is added to an existing shape, it does not matter which periphery edge is used in the path merge.
For example, when box 5 is added to the shape in Fig.~\ref{fig:random_box_add}, the periphery edge in box 2 was used to merge with the unit path in box 5. 
The nearest periphery edges in box 1 or box 4 could have also been used and the result would still have been a planar Hamiltonian cycle.

Finally, because path merging is used in Proof~\ref{claim_random} to merge unit paths when each additional box is added to the shape, the resulting planar Hamiltonian cycles have two particular properties.
First, each non-spanning edge is directed clockwise around the center of its box because each non-spanning edge originated in a unit path which is clockwise by definition.
Second, each spanning edge is part of a pair of parallel and opposite edges that span across the same box connection.
In fact, if a tree data representation were used to track the order in which boxes were added to the shape, one would see that these pairs span from parent to child in the tree.
For example, box 4 is a child of box 3 when it is added to the shape in Fig.~\ref{fig:random_box_add}e, and a pair of parallel and opposite edges span this parent-child box connection.

We need a term to differentiate planar Hamiltonian cycles that have these two distinct properties from other planar Hamiltonian cycles that do not. 
In addition, we know from practice that planar Hamiltonian cycles are broken at adjacent start and end nodes to facilitate the movement of robots to and from a charging station. 
Thus, we define a \textbf{valid path} as a planar Hamiltonian cycle with adjacent entry and exit nodes on the periphery of the shape such that each non-spanning edge is directed clockwise around the center of its box and each spanning edge is part of a pair of parallel and opposite edges spanning across the same box connection.

%%%%%%% VALID PATHS HERE %%%%%%%%%%%%%%%%%%
% A valid path is a planar Hamiltonian cycle where each non-spanning edge is directed clockwise around the center of its box, and each spanning edge has a adjacent, parallel, and opposite edge spanning across the same box side.
% A valid path can be constructed by adding boxes together in an arbitrary sequence such that each box addition maintains a valid shape and that each box's unit path is attached to the existing path via the PM operation.
% In this way, it can be shown that a valid path can be found for any valid shape (REF HERE).

% Six examples of valid paths are shown for six different shapes in Fig.~\ref{fig:random_box_add}a-f, each one the result of a different arbitrary box addition. 
% For instance, when box ``2" is added to boxes ``0" and ``1" in Fig.~\ref{fig:random_box_add}c, the resulting shape is valid, and the unit path through box ``2" is connected to the existing path through ``0" and ``1" via the PM operation such that there are now two adjacent, parallel, and opposite edges spanning between box ``0" and box ``2."
% The red edges indicate the next locations at which another box could be added and its unit path attached via the PM operation. 
% Since these edges connect two periphery nodes, they are defined as \textbf{periphery edges}. 
% Periphery edges are non-spanning by definition.

\subsection{Planar Hamiltonian Cycles via the Default Behavior}\label{sec:default_makes_preferred}
Earlier we claimed that the default behavior always results in a planar Hamiltonian cycle (for any valid shape).
As before, to prove this claim, we have to introduce a new concept: ordered assembly of valid shapes via the \textbf{depth-first clockwise-priority (DFCP) method}.
Also, we will define a \textbf{preferred path} as a planar Hamiltonian cycle that is generated by the default behavior and a \textbf{DFCP path} as a planar Hamiltonian cycle that is generated by the  DFCP method (and later, we will show that preferred paths and DFCP paths are identical).

We established in Section~\ref{sec:valid_shapes} that a valid shape can be constructed via the assembly of boxes in an arbitrary order.
We then constrained the assembly such that each additional box had to maintain a valid shape, and we found that this resulted in valid paths.
However, the order in which boxes were assembled remained arbitrary; we only enforced the valid shape constraint.
If we restrict the order in which boxes are assembled while enforcing the valid shape constraint, then we can create a subset of valid paths (called DFCP paths) since a sequence of box additions in a particular order is a subset of all box additions in arbitrary orders.

\begin{figure*}[!t]
\centering
\includegraphics[width=0.9\textwidth]{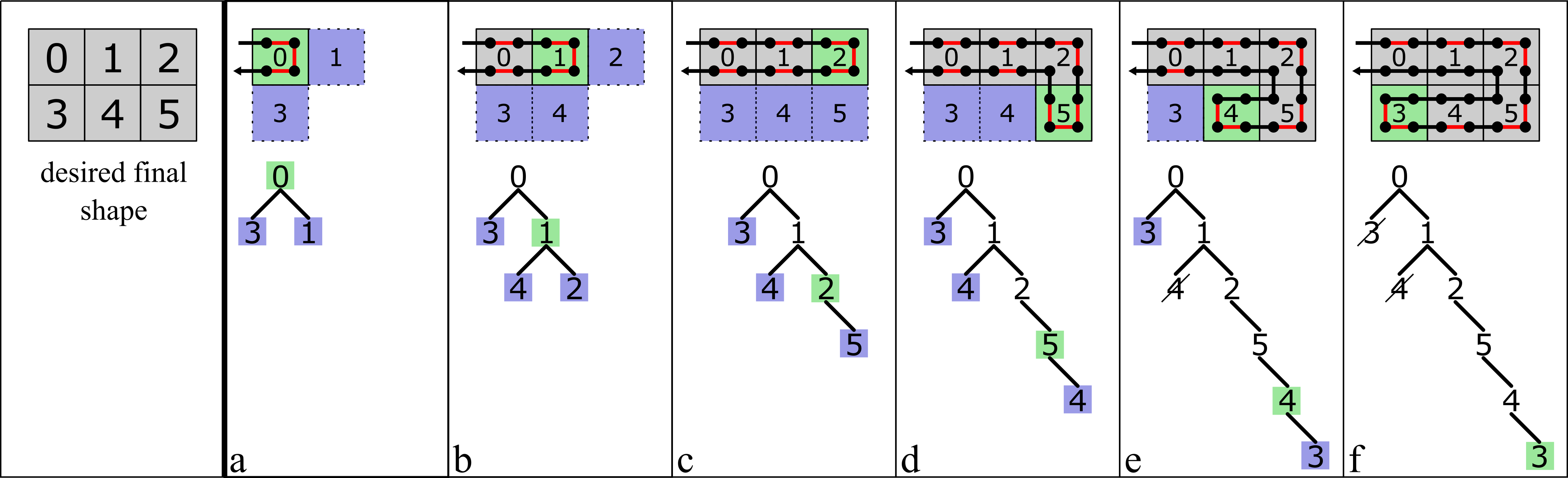}
\caption{(left) A desired final shape of 6 boxes and (right) a sequence of temporary shapes and corresponding tree data structures to produce the desired final shape (a-f).  Periphery edges are in red. Arrows indicate path direction.  
Green identifies the most recently added box to the shape, and blue identifies the frontier.}
\label{fig:cwp}
\end{figure*}

In particular, in the DFCP method, we restrict the order in which boxes are assembled to mimic that of a classical depth-first search.
Every time a box is added to a shape, we update the frontier of plausible next boxes.
The frontier is updated with all boxes that will be in the desired final shape, that have not already been added to the shape, and that share a side with the box most recently added to the shape (this ensures each box addition results in a valid shape).
Any box appended to the frontier is also incorporated in a tree data structure representation of the boxes identified thus far.
In the case where multiple boxes are appended to the frontier at the same time, each box is included in the tree right to left corresponding to a clockwise order.
In other words, an imaginary vector is drawn from the center of the box most recently added to the shape towards the center of that box's parent in the tree (or to the periphery of the shape in the case of the root box).
That vector is swept clockwise around the center of the most recently added box.
As the vector is swept, boxes in the frontier are added to the tree right to left in the order in which the vector points towards those boxes.
For example, boxes 2 and 4 are added to the tree data structure from right to left in Fig.~\ref{fig:cwp}b because a vector centered at box 1 and pointing toward box 0 sweeps clockwise first through box 2 and then through box 4.

Once all boxes in the frontier have been added to the tree, the deepest unexplored node in the tree will be added to the shape next.
In the case of a tie for deepest unexplored node, boxes are added to the shape from right to left.
This results in boxes added to the shape in accordance with a classical depth-first search where ties are broken in a clockwise-priority manner (thus the name: depth-first clockwise-priority method).
Once a box is added to the shape, all other instances of it are removed from the frontier.
The process continues until there are no other potential boxes to add, and the shape matches the desired final shape (e.g., Fig.~\ref{fig:cwp}).

After each addition to the shape, the path is adjusted in the same way as Proof~\ref{claim_random}: the unit path of the new box is merged with the existing path via path merging. 
However, unlike before, in the DFCP method we constrain the merge so that the periphery edge of the existing path is in the parent of the new box.
For example, when box 3 is added to the shape in Fig.~\ref{fig:cwp}f, the periphery edge in box 4 is merged with the unit path of box 3 instead of the periphery edge through box 0 because box 4 is the parent of box 3 in the tree.

The DFCP method results in DFCP paths (e.g., Fig.~\ref{fig:cwp}f).
As a subset of valid paths, DFCP paths are also planar Hamiltonian cycles with the two unique properties of valid paths: clockwise non-spanning edges and spanning edge pairs between parent and child boxes.
The only difference is the order in which the path navigates through boxes.
With this established, we can now prove our claim, and in the process we will establish two additional lemmas.

\begin{lemma}\label{lem:influence_up}
    A robot's path is not influenced by the presence of a box until the robot reaches that box.
\end{lemma}

\begin{lemma}\label{lem:influence_down}
    A robot's path is not influenced by the presence of a box after the robot has visited all nodes within that box.
\end{lemma}

\begin{theorem}\label{claim_default}
    The default behavior always results in a planar Hamiltonian cycle for any valid shape.
\end{theorem}

\begin{proof}[\proofname\ 2]
This is a proof by induction that will show that the DFCP path produced by the DFCP method and the preferred path produced by the default behavior are identical for any valid shape as that shape is built up one box at a time.

Consider the base case of the unit box ($\mathcal{B}=1$).
The DFCP path is the unit path with an entry and exit node (e.g., Fig.~\ref{fig:cwp}a). 
Now, consider the preferred path formed by the default behavior.
After entering the shape, the robot never sees any new boxes or previously visited boxes (there are none!), so only Algorithm~\ref{alg:default_behavior} rule 2 applies.
The robot moves from node to node in a clockwise manner until it reaches the exit node and exits the shape.
The resulting path is identical to the DFCP path.

Next, assume there exists a shape of $\mathcal{B}=K$ boxes, where $K$ is some integer and $K\geq1$.
Assume that the shape has a DFCP path that matches the preferred path generated by the default behavior.
One additional box (call it box $Y$) is added to the shape in accordance with the DFCP method for a total of $\mathcal{B} = K + 1$ boxes.
A path merge between the existing $\mathcal{B}=K$ path and the unit path of box $Y$ will result in a new DFCP path that passes through box $Y$ via a 5-edge loop between the two nodes nearest to (and in the parent of) box $Y$: one spanning edge into box $Y$, three non-spanning edges within box $Y$, and one spanning edge back into the parent of box $Y$.
No other portions of the preferred path will be changed from the $\mathcal{B} = K$ case.
For example, see the addition of box 5 to the shape in Fig.~\ref{fig:cwp}d.

A new preferred path generated by the default behavior will match this new DFCP path because the DFCP method and the default behavior explore boxes in the same manner. 
For the DFCP method, adding box $Y$ to the shape can be thought of as expanding the node in the tree that is the deepest and rightmost unexplored node (where rightmost is analogous to clockwise-priority).
Similarly, when a robot is executing the default behavior, it traverses the shape in a depth-first manner due to the priority of Algorithm~\ref{alg:default_behavior} rule 1 over Algorithm~\ref{alg:default_behavior} rule 2. 
When this depth-first precedent is combined with the clockwise priority of Algorithm~\ref{alg:default_behavior} rule 2, one can conclude that a robot will see and enter neighboring boxes in a depth-first clockwise-priority order.

Since the DFCP method and the default behavior explore boxes in the same order, box $Y$ is both the deepest and rightmost unexplored node in the DFCP method and the last unvisited box the robot reaches when traversing the $\mathcal{B}=K+1$ shape via the default behavior.
This means the robot's path is not influenced by the presence of box $Y$ until the robot reaches box $Y$ (this gives us Lemma~\ref{lem:influence_up}).
Thus, the DFCP path and the preferred path are unchanged from the $\mathcal{B}=K$ case (and identical to one another) for all portions of the path prior to box $Y$.
Upon reaching box $Y$, the robot executing the default behavior will follow Algorithm~\ref{alg:default_behavior} rule 1 to enter box $Y$. 
When in box $Y$, there are no other unvisited boxes, so Algorithm~\ref{alg:default_behavior} rule 1 will not trigger for the rest of the robot's motion through the shape.
The robot will move clockwise in box $Y$ (Algorithm~\ref{alg:default_behavior} rule 2) and then move into the parent of box $Y$ (Algorithm~\ref{alg:default_behavior} rule 3).
This is the same as the 5-edge loop that the DFCP path takes through box $Y$ back to the parent of box $Y$.
Finally, for the remainder of the shape, the robot will fill out previously visited boxes using Algorithm~\ref{alg:default_behavior} rule 2, and then move up the tree to parent boxes using Algorithm~\ref{alg:default_behavior} rule 3. 
This means after the robot exits box $Y$, the presence of box $Y$ does not effect the behavior in any way (this gives us Lemma~\ref{lem:influence_down}).
It also means that the robot's path (via the default behavior) after box $Y$ is the same as in the $\mathcal{B}=K$ case which, more importantly, is also the same as the DFCP path.
So, in summary, the DFCP path and the preferred path are identical for all portions of the path prior to, within, and after box $Y$.

Since the DFCP path and the preferred path are identical for the case of $\mathcal{B}=1$ and for the transition from $\mathcal{B}=K$ to $\mathcal{B}=K+1$ for some $K\geq1$, we can conclude by induction that the default behavior will always produce a preferred path identical to the DFCP path for any valid shape of $\mathcal{B}\geq1$.
And, since we know DFCP paths are a type of planar Hamiltonian cycle, we know Theorem~\ref{claim_default} is valid.
Further, we know preferred paths are a subset of valid paths with the two unique properties of valid paths: clockwise non-spanning edges and spanning edge pairs between parent and child boxes.
\end{proof}

\section{Shape Adaptability}\label{sec:shape_adaptability}
Once a swarm has formed a persistent shape in space, the swarm must remain adaptable to shape changes (the addition or removal of a box from the shape).
The robots must reconcile these changes using only local information and making only local path modifications since no single robot can have global influence over the swarm.
To do this, we simplify adaptability into three basic steps: 1) detection,  2) primary changes to fill in a new box or empty a removed box, and 3)~secondary changes to return the swarm to the preferred path.

To explain these steps, we will use \textbf{downstream} to refer to a robot that has previously visited a particular grid node, and \textbf{upstream} to refer to a robot that has yet to visit a particular grid node.
For example, in Fig.~\ref{fig:primary_add}a, robots 6 through 11 are all said to be upstream of robot 5's position, and robots 0 through 4 are said to be downstream of robot 5.
We will also use \textbf{existing shape} and \textbf{existing path} to refer to the shape and path that existed prior to the change and \textbf{new shape} and \textbf{new path} to refer to the post-change shape and post-change path, respectively.
The term \textbf{interim path} will be used to refer to any provisional path formed in the process of changing from an existing path to a new path.

\subsection{Change Detection}
% POINT: Detecting change is hardware specific.
% Regardless of how it is detected, it is communicated through the swarm (early text)
% Human can't change the entry and exit positions
The first step in shape adaptability, detecting the change, is largely outside of the scope of work in this paper because the exact process by which a swarm detects an external environmental stimulus is often dependent on robot hardware. %, and robot hardware is often dependent on a particular swarm implementation. 
To keep this work hardware agnostic, detection is simplified by using human gestures as environmental stimuli and assuming that robots can sense two unique gestures: one indicating a box addition and one indicating a box subtraction.
Further, we assume that gesture identification is local (only robots nearest to the gesture can sense it), and that a gesture occurs at the location of the desired change.
Once a gesture is sensed, robots can propagate notification of the change across the swarm using robot-to-robot communication.
It is assumed that a human can initiate a change to any portion of the existing shape that results in a new valid shape.
This includes creating and removing holes but excludes changing the entry and exit nodes.
It is also assumed that a new change is only initiated after an existing change is completely resolved.

\subsection{Primary Changes}
% POINT: Local changes: fill or remove. Identify the one to enter. Enter the shape, 4 robots follow it in a clockwise motion before reconnecting with the existing path. Nothing else changes. Treated like a leaf on the tree. Once filled, global change initiated.
After a change has been detected and communicated, the swarm begins making changes to the path.
In the case of a box addition, this involves robots filling in the newly added box.
In the case of a box subtraction, this involves robots filing out of the newly removed box.
These behaviors are captured in Algorithm~\ref{alg:primary_changes} and described in this section.

\begin{figure*}[!t]
\centering
\includegraphics[width=0.8\textwidth]{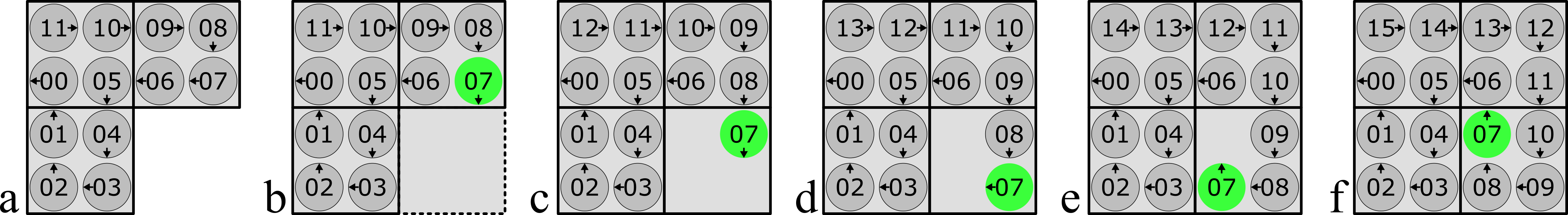}
\caption{Primary changes in response to a box added to the shape in (a). Numbers represent robot IDs. Arrows indicate robot heading. Robot 7 (green) is at the point of inflection in (b) and then is the first robot into the new box.}
\label{fig:primary_add}
\end{figure*}

\begin{figure*}[!t]
\centering
\includegraphics[width=0.8\textwidth]{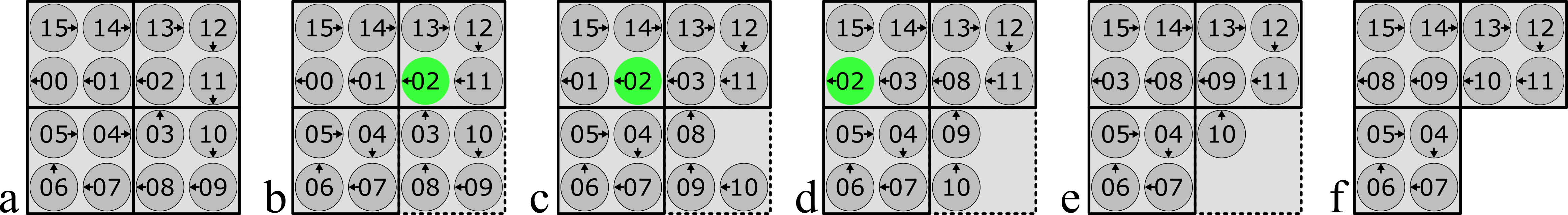}
\caption{Primary changes in response to a box subtracted from the shape in (a). Numbers represent robot IDs. Arrows indicate robot heading. Robot 2 (green) is at the point of inflection in (b) and then leads robots 3, 8, 9, and 10 out of the removed box.}
\label{fig:primary_subtract}
\end{figure*}

In the case of addition, the swarm first identifies the critical point for the change ($n_{cp}$).
This is identified by the robot in the existing shape that has visited only one grid node adjacent to the new box and was planning to take a non-spanning periphery edge as its next move. 
The grid node occupied by that robot becomes the point of inflection for communicating and addressing the change because all robots downstream of that position may be affected by the change and all robots upstream of that position will not be affected by the change (by Lemma~\ref{lem:influence_up}).
Once $n_{cp}$ is identified, each robot takes action in accordance with Algorithm~\ref{alg:primary_changes} lines 1-8 depending on if the robot is upstream, downstream, or at $n_{cp}$.
Once all four nodes in the new box are occupied, the robots have filled the new shape, the interim path is a valid path (Theorem~\ref{claim_add_valid} proven later in Section~\ref{sec:addition_theory}), and primary changes are complete.

For example, consider a three-box shape with a preferred path such as the one drawn in Fig.~\ref{fig:primary_add}a. 
Assume a fourth box is added to the shape to make a square of four boxes (Fig.~\ref{fig:primary_add}b).
Robot 7, shown in green, is occupying the point of inflection since it has only visited one grid node adjacent to the new box, and, prior to the addition, it intended to take a non-spanning periphery edge for its next move (see the heading of robot 7 in Fig.~\ref{fig:primary_add}a).
With robot 7 at the $n_{cp}$, the downstream robots (robots 0 through 6) pause their motion until the new box is filled.
Meanwhile, robots numbered 7 and up continue to move along the upstream path via the default behavior and fill in the new box via clockwise non-spanning edges (Fig.~\ref{fig:primary_add}c-f).
Once the new box is filled, primary changes are complete and secondary changes can begin.

\begin{algorithm}[t]
\caption{Primary Changes}
\label{alg:primary_changes}
    \begin{algorithmic}[1]
        \Require $CT$, $n_{cp}, V_n$ \Comment{$CT$ = change type}
        \If{$CT$ is addition}
            \If{$n_{cp} \in V_n$} \Comment{downstream case}
                \State pause motion and ignore change
            \ElsIf{robot is at $n_{cp}$} \Comment{$n_{cp}$ case}
                \State fill in new box via clockwise motion
            \Else \Comment{upstream case}
                \State execute default behavior until reaching $n_{cp}$
            \EndIf
        \ElsIf{$CT$ is subtraction}
            \Statex \hspace{1em} $\triangleright$ \textbf{Set Boolean True if robot is currently }
            \Statex \hspace{1em} $\triangleright$ \textbf{within the removed box or False if not}
            \State $inRemovedBox \gets$ checkInRemovedBox()   
            \If{\textbf{not} $inRemovedBox$}
                \If{$n_{cp} \in V_n$} \Comment{downstream case}
                    \State execute default behavior
                \ElsIf{robot is at $n_{cp}$} \Comment{$n_{cp}$ case}
                    \State execute default behavior
                \ElsIf{$n_{cp} \notin V_n$} \Comment{upstream case}
                    \State pause motion
                    \If{robot planned to enter removed box}
                        \State plan clockwise non-spanning motion
                    \EndIf
                \EndIf
            \Else
                \State exit box via clockwise motion to $n_{cp}$
            \EndIf
        \EndIf 
    \end{algorithmic}
\end{algorithm}

%POINT: Follow out is similar.
% Identify the one to follow out. 4 robots cycle out of the current box in a clockwise manner. Last robot out initiates start of global changes.
The case of box subtraction is similar to addition, but there are some key differences, especially with respect to how the existing path is impacted by the box removal.
A subtraction will immediately impact robots in two categories: 1) robots within the removed box and 2) robots that have passed through the removed box but have not yet visited all of the nodes within the removed box.
This latter category of robots would have to move through the removed box again to exit the shape via the existing path.
Robots that have yet to reach the removed box, and robots that have already visited every node in the removed box, will not be immediately impacted by its removal (by Lemmas~\ref{lem:influence_up} and~\ref{lem:influence_down}, respectively). 
Therefore, as the change is communicated across the swarm, a point of inflection node ($n_{cp}$) is identified like before, but now it represents the earliest node in the existing path after which the subtraction will have no effect.
This allows robots within the removed box to safely follow the robot at the point of inflection out of the removed box without impacting other portions of the path.
The point of inflection is identified as the node occupied by the robot with following criteria:
\begin{enumerate}
    \item The robot is at a node adjacent to the removed box.
    \item The robot has previously visited all four nodes in the removed box.
    \item The robot's previous node was within the removed box.
\end{enumerate}
\noindent 

Once a point of inflection has been identified, each robot takes action in accordance with Algorithm~\ref{alg:primary_changes} lines 10-24 depending on the robot's relative position to the removed box and $n_{cp}$ (upstream, downstream, etc.).
If any of the upstream robots are in a box adjacent to the box to be removed and had planned to move into the box prior to its removal, they adjust their plan to instead move clockwise within their current box.
In other words, they adjust their path from a spanning edge to a non-spanning edge. % (which is similar to the process of path separation).
Once all of the robots have exited the box to be removed, primary changes are finished and secondary changes can begin.

For example, consider a shape of four boxes in a 2x2 array with a preferred path cycling robots through the shape (Fig.~\ref{fig:primary_subtract}a).
Assume the swam identifies a gesture indicating the removal of the lower right box (Fig.~\ref{fig:primary_subtract}b).
In such a case, robot 2 would be identified as residing at the point of inflection since it is adjacent to the removed box, it has visited every node in the removed box, and its previous node was within the removed box.
With robot 2 at the point of inflection, robots 0, 1, and 2 can continue to move through the shape as if the removal never occurred.
Robots upstream of 2's position that are not within the removed box (i.e., robots 4-7 and 11-15) remain still as the box is emptied.
Furthermore, robot 4 and robot 11, which had previously planned to take a spanning edges into the removed box, adjusts their paths to take non-spanning edges within their own boxes (Fig.~\ref{fig:primary_subtract}a-b).
Finally, robots 3, 8, 9, and 10 follow robot 2 out of the removed box until the box is empty (Fig.~\ref{fig:primary_subtract}c-f).

Unlike the case of addition, however, the resulting path is not necessarily valid after a subtraction. 
Primary changes for subtraction can only guarantee that the resulting path is at least pseudo-valid (Theorem~\ref{claim_sub_pseudovalid} proved later in Section \ref{sec:subtraction_theory}).
Although we will discuss pseudo-valid paths at length in later sections, we will use the following working definition for now: a \textbf{pseudo-valid path} is a discontinuous path composed of multiple planar Hamiltonian cycles that is created by processes similar to path separation or path merging. 
Examples of a pseudo-valid paths can be found in Fig.~\ref{fig:movement_overview}b and Fig.~\ref{venn}c.

% \begin{figure*}[!t]
% \centering
% \includegraphics[width=0.9\textwidth]{primary_changes_subtract.png}
% \caption{The sequence of steps that occurs when a box is removed from the shape in (a). Numbers represent unique IDs of the robots at each node in the shape. Arrows on the robots indicate heading along the path. The green robot is at the point of inflection in (b) and then leads the other robots out of the removed shape.}
% \label{fig:primary_subtract}
% \end{figure*}

\subsection{Secondary Changes}
%POINT: global changes are important for persistance.
Secondary changes are the series of local path changes that convert the swarm from following the post-primary-changes path to the preferred path of the new shape.
As the last step in shape adaptability, secondary changes are extremely important for shape persistence because they guarantee that the swarm is following the preferred path of the new shape.
If the swarm did not return to moving along the preferred path of the new shape after each change, then the impact of each change would have to be communicated to the swarm indefinitely. % since local changes do not always result in the preferred path for the new shape. 
In those cases, the swarm's path through the shape would forever be altered by the shape changes encountered in the past.
By always reverting to the preferred path, the swarm behavior becomes independent of its past, and the swarm can continue to execute the default behavior to persist in the shape while remaining adaptive to future changes.

% Global changes done in 2 ways.
% Comm-based (proof)
% Movement-based (proof)
There are two interchangeable methods for making secondary changes.
The first method, called the \textbf{communication-based method}, resolves secondary changes by passing a message robot-to-robot through the swarm along the new preferred path immediately after primary changes are complete.
Each swarm member reacts to this message and, if necessary, makes a local change to its planned path to match the new preferred path. 
The swarm then executes the default behavior and moves along the new preferred path as soon as the message has traveled through the shape.
The second method, called the \textbf{movement-based method}, resolves secondary changes by promoting a single robot to make a series of local path changes as it moves through the remainder of the new shape.
This results in a sequence of interim paths until the swarm finally converges to the new preferred path.

The communication-based method is best suited for applications where communication speed is much faster than the traveling speeds of the robots and where secondary changes must be completed quickly (within one time period, $\tau$). 
This is because secondary changes are complete as soon as the message reaches the robot at the exit node, and one does not have to wait for a robot to travel through the remainder of the new shape (as is the case with the movement-based method).
By contrast, the movement-based method is better suited for applications where communication speed is slower, and a message is not likely to be passed through the swarm in one time period ($\tau$).
In practice, a swarm would be designed to execute one of the two methods for secondary changes for the duration of its task; the swarm would not alternate between the two methods during any given task.

Both methods begin after primary changes have been addressed, and both start from the same grid node in the new shape.
For clarity, we define the node from which secondary changes begin as the \textbf{secondary change start node (SCSN)}.
It represents the latest node in the path where all upstream robots will not be affected by the change (by Lemma~\ref{lem:influence_up}).
In the case of box addition, the point of inflection identified during primary changes (prior to filling in the new box) and the SCSN are identical.
In the case of box subtraction, the SCSN is the node occupied by the robot with the following criteria:
\begin{enumerate}
    \item The robot is at a node adjacent to the removed box.
    \item The robot has never visited any of the nodes in the removed box.
    \item Prior to removal, the robot's next position would have been in the removed box.
\end{enumerate}
In the subtraction example, the SCSN is occupied by robot 11 in  Fig.~\ref{fig:primary_subtract}.

\subsubsection{Communication-Based Method}
In this method, the robot occupying the SCSN initiates a \textbf{memory message} through the swarm. 
The memory message is essentially a virtual robot executing the default behavior along the new preferred path.
Instead of physically moving, the message is transmitted from robot to robot along the new preferred path in accordance with Algorithm~\ref{alg_mem_msg}.
Each recipient of the memory message re-writes its boxes visited ($V_B$) and nodes visited ($V_n$) to match that of the memory message ($m$). 
Then, the robot appends its own grid node and box information to the data structures before executing the default behavior without actually moving.
By doing this, the robot determines its next position.
The robot occupying the original recipient's next position is the next recipient of the updated memory message ($m'$). 
The original recipient then transmits the updated memory message to the next recipient and the process continues.

\begin{algorithm}[t]
\caption{Memory Message Reception}
\label{alg_mem_msg}
    \begin{algorithmic}[1]
        \Require $m,S,V_B,V_n$ \Comment{$m$ = memory message}
        \Ensure $m'$ \Comment{$m'$ = new memory message}
        \Statex $\triangleright$ \textbf{Update data structures with information from $m$}
        \State $V_B \gets V_B$ from $m$ and $b$ \Comment{$b$ = current box}
        \State $V_n \gets V_n$ from $m$ and $n$ \Comment{$n$ = current node}
        \Statex $\triangleright$ \textbf{Determine next edge via default behavior}
        \State $\epsilon' \gets$ defaultBehavior($b,S,V_B,V_n$) 
        \Statex $\triangleright$\textbf{ Determine next node}
        \State $n' \gets $ node at end of $\epsilon'$\Comment{$n'$ = next node}
        \Statex $\triangleright$ \textbf{Create message to send to robot at next node}
        \State $m' \gets $ createMessage($V_B, V_n, n'$) 
    \end{algorithmic}
\end{algorithm}

% In this way, each recipient believes it reached its current location via the preferred path of the new shape.
This creates local changes in the path as the message traverses through the swarm.
Each recipient adjusts its memory, effectively rewriting its past as if it had always been moving along the new preferred path of the new shape.
This ensures that all future robot movements will also be along the new preferred path of the new shape.
The memory message is only sent downstream from the SCSN because all upstream portions of the path are the same for both the existing path and the new path (by Lemma~\ref{lem:influence_up}). %since any necessary global changes will only occur for robots that have previously visited (in the case of subtraction) or should have previously visited (in the case of addition) the changed box.
Eventually, the memory message reaches the robot at the exit node and the robots resume the default behavior as if nothing ever happened.

Although the communication-based method works for both shape additions and subtractions, we will consider only the example subtraction case (from Fig.~\ref{fig:primary_subtract}) for brevity.
After the box is removed and primary changes are resolved, robot 11 (at the SCSN) begins transmitting a memory message through the swarm.
Fig.~\ref{fig:message_based} shows the state of the swarm when memory message transmission begins (Fig.~\ref{fig:message_based}a) and  ends (Fig.~\ref{fig:message_based}b). 
Note Fig.~\ref{fig:message_based}a is identical to Fig.~\ref{fig:primary_subtract}f because secondary changes begin where primary changes left off.
In accordance with Algorithm~\ref{alg_mem_msg}, the memory message moves through the swarm along the new preferred path, so it passes through robots in the following order: 11, 10, 9, 4, 7, 6, 5, 8.
Robots 11, 10, 4, 7, 6, and 8 do not change their heading direction after processing the memory message because their heading in the interim path (Fig.~\ref{fig:message_based}a) is identical to the new preferred path (Fig.~\ref{fig:message_based}b).
However, robots 9 and 5 both make a local change to their heading in response to the memory message so that the swarm is following the new preferred path (Fig.~\ref{fig:message_based}a to b).

\begin{figure}[t]
\centering
\includegraphics[width=52mm]{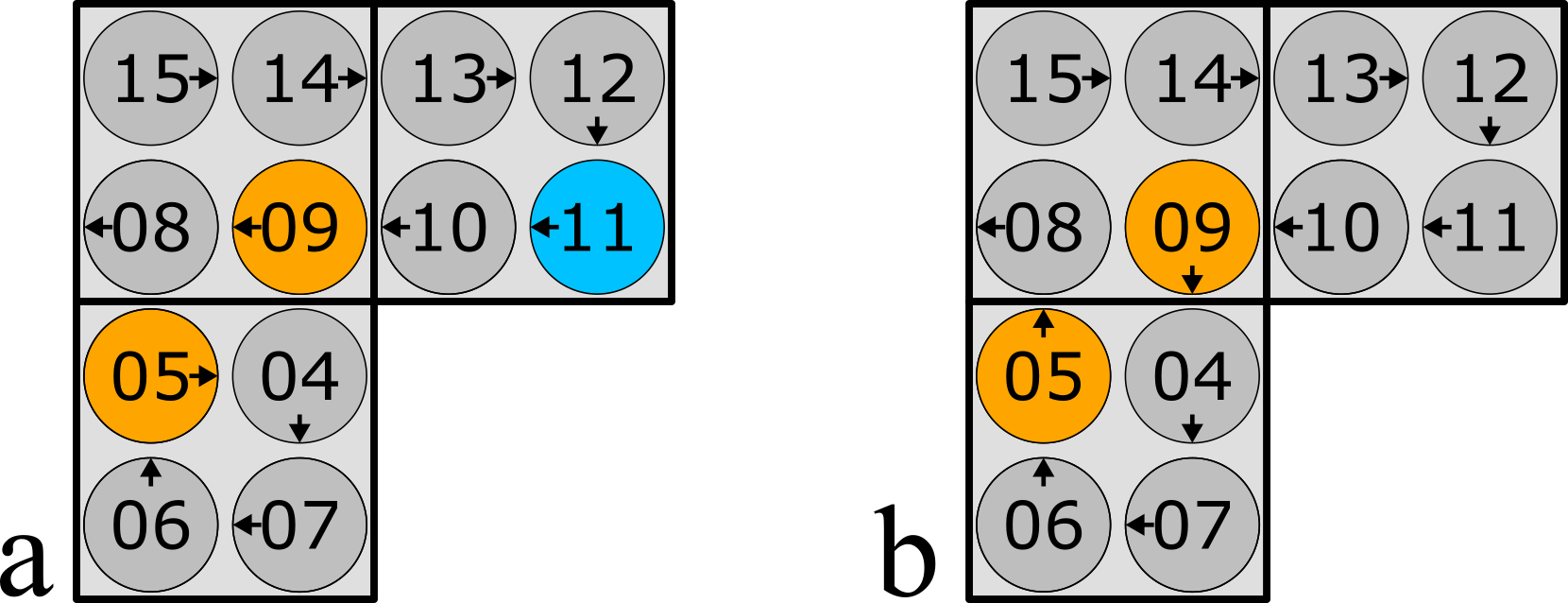}
\caption{Secondary changes. (a) The state of the swarm immediately following the primary changes (Fig.~\ref{fig:primary_subtract}). Robot 11 is at the SCSN (blue). (b) The state of the swarm after the communication-based method is complete with robots 5 and 9 (orange) having changed their heading. Numbers represent unique IDs of the robots. Arrows indicate robot heading.}
\label{fig:message_based}
\end{figure}

\subsubsection{Movement-Based Method}
In this method, secondary changes happen more slowly than in the communication-based method.
This is because they rely on a robot moving through the shape, swapping destinations with other robots in the swarm to make local path changes as it goes.
This results in a series of interim paths from the path immediately after primary changes complete to the new preferred path.
For example, the path after primary changes for the box addition described in Fig.~\ref{fig:primary_add} looks like the path drawn in Fig.~\ref{fig:movement_overview}a. 
This path is changed twice in a sequence of interim paths from Fig.~\ref{fig:movement_overview}a to Fig.~\ref{fig:movement_overview}c in order to achieve the new preferred path (Fig.~\ref{fig:movement_overview}c).

\begin{figure}[t]
\centering
\includegraphics[width=70mm]{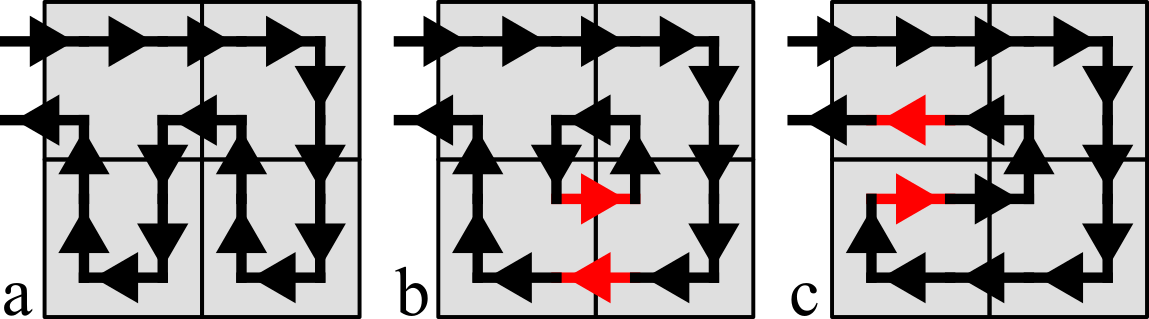}
\caption{A sequence of interim paths leading from an interim path (a) to the new preferred path (c). Changes from the previous path are shown in red. Arrows indicate edge direction.
}
\label{fig:movement_overview}
\end{figure}

To begin converting the swarm's path to the new preferred path, the swarm promotes a \textbf{change robot}.
The change robot is not inherently special; it is just the robot that happens to be present at the SCSN when primary changes are completed.
As before, all portions of the path upstream of the SCSN are the same for both the existing path and the new preferred path (by Lemma~\ref{lem:influence_up}), so all impacts of the change robot will come after the change robot's initial position. 
Thus, the change robot represents the first robot capable of traversing the preferred path of the new shape using only the default behavior.
% As the change robot moves through the shape, it executes its default behavior and communicates its desired next node to its neighbors.

Even though the change robot is executing the default behavior, robots downstream of its position may not be.
Since the post-primary-changes interim path is not necessarily the preferred path (it could be valid or pseudo-valid as proven later in Proofs~\ref{claim_add_valid}~and~\ref{claim_sub_pseudovalid}, respectively), robots downstream of the change robot's position may be executing non-default behavior to maintain a non-preferred interim path.
Non-default behavior is facilitated by \textbf{pass-back robots}: robots that use neighbor-to-neighbor communication to ``pass back" a non-default movement such that upstream robots follow the pass-back robot's non-default path. 

Robots involved in or adjacent to a primary change (i.e., those that fill a new box and those that are adjacent to a newly added or removed box) immediately become pass-back robots when primary changes are finished.
Then, after they move to their next grid node, each pass-back robot transmits a \textbf{pass-back message} (Algorithm~\ref{alg_passback}) to the robot in its previous grid node ($n_{prev}$).
The recipient of the pass-back message then adjusts its path to follow the pass-back robot (Algorithm~\ref{alg_passback_receive}).
The recipient also rewrites its memory ($V_B$ and $V_n$) to match that of the sender (the original pass-back robot) and then becomes a pass-back robot itself. 
This ensures that the grid node becomes a semi-permanent instance of non-default behavior because the recipient will pass-back the same message to the next robot in line after it moves.
Any upstream robot that receives a pass-back message will follow the robot that previously occupied the grid node and then tell the next upstream robot to do the same.

After transmitting a pass-back message, a pass-back robot converts back to a regular robot and follows the default behavior.
However, if it ever receives another pass-back message from a downstream robot, it will become a pass-back robot again (per Algorithm~\ref{alg_passback_receive}).

Pass-back behavior is only cleared from the swarm by the change robot. 
It does this by moving through downstream grid nodes and ignoring pass-back messages as it moves (i.e., the change robot does not continue to transmit pass-back messages even if a pass-back robot attempts to send a message to the change robot).

\begin{algorithm}[t]
\caption{Pass-back Robot Behavior: Transmission}
\label{alg_passback}
    \begin{algorithmic}[1]
        \Require $\epsilon', V_n, V_B$
        \Ensure $m_{pb}$ \Comment{$m_{pb}$ = pass-back message}
        \State move to next node via $\epsilon'$
        \State $n_{prev} \gets V_n[-1]$ \Comment{get previous node}
        \State append $b$ to $V_B$ \Comment{update boxes visited w. current box}
        \State append $n$ to $V_n$ \Comment{update nodes visited w. current node}
        \Statex $\triangleright$ \textbf{Create message to send to previous node}
        \State $m_{pb} \gets $ createMessage($V_B$, $V_n$, $\epsilon'$, $n_{prev}$)
        
    \end{algorithmic}
\end{algorithm}

\begin{algorithm}[t]
\caption{Pass-back Robot Behavior: Reception}
\label{alg_passback_receive}
    \begin{algorithmic}[1]
        \Require $m_{pb}$ \Comment{$m_{pb}$ = pass-back message}
        \State $V_B \gets V_B$ from $m_{pb}$\Comment{overwrite boxes visited}
        \State $V_n \gets V_n$ from $m_{pb}$\Comment{overwrite nodes visited}
        \State $\epsilon' \gets \epsilon'$ from $m_{pb}$ \Comment{overwrite edge to next node}
        \State become pass-back robot         
    \end{algorithmic}
\end{algorithm}

Since its downstream neighbors may not be navigating along the same path as the change robot, conflicts can occur when both the change robot and a neighboring robot intend to move to the same node.
% If any of its neighbors intend to go to the same destination node as the change robot, the conflicting neighbor will change its next node via a process called \textbf{destination swapping}.
Such conflicts are resolved via a process called \textbf{destination swapping}.
In a destination swap, the change robot effectively has a ``choice" to follow the default behavior along an edge to a node (destination 1) or to follow the current interim path along an edge to some other node (destination 2).
Since the change robot is always executing its default behavior, it will always choose destination 1.
This leaves destination 2 available for the conflicting neighbor robot that had previously planned to move to destination 1 but can no longer since the change robot has priority.
Thus, the change robot effectively ``swaps destinations" with the conflicting neighbor robot, giving the conflicting neighbor robot destination 2 in exchange for destination 1.
In practice, destination swapping is facilitated by the change robot constantly broadcasting its planned movements so that conflicting neighboring robots can react to the message and change their destination before collisions occur.

After destination swapping, the robot that changed its destination to accommodate the change robot becomes a pass-back robot so that other robots continue to follow this new direction and a new interim path is created.
This process continues with each destination swap creating a sequence of interim paths until the new preferred path is achieved (Theorem~\ref{claim_swap} proved later in Section~\ref{sec:secondary_changes_proof}).
Once the change robot reaches the end of the shape, all pass-back nodes have been cleared, the swarm is executing the new preferred path, and all future robots can continue to execute their default behavior without exception.
Finally, the change robot reverts back to a normal swarm robot as it departs for the charging station.

Similar to the communication-based method, the movement-based method transforms interim paths to preferred paths via an identical process for both addition and subtraction.
For brevity, only the addition example from Fig.~\ref{fig:primary_add} will be discussed here.
After addition, the swarm has formed the interim path shown in Fig.~\ref{fig:primary_add}f.
For clarity, a graphic showing the sequence of events for secondary changes is provided in Fig.~\ref{fig:alg3_intro}.
As before, secondary changes begin where primary changes left off, so Fig.~\ref{fig:primary_add}f matches Fig.~\ref{fig:alg3_intro}a. 

\begin{figure}[t]
\centering
\includegraphics[width=0.42\textwidth]{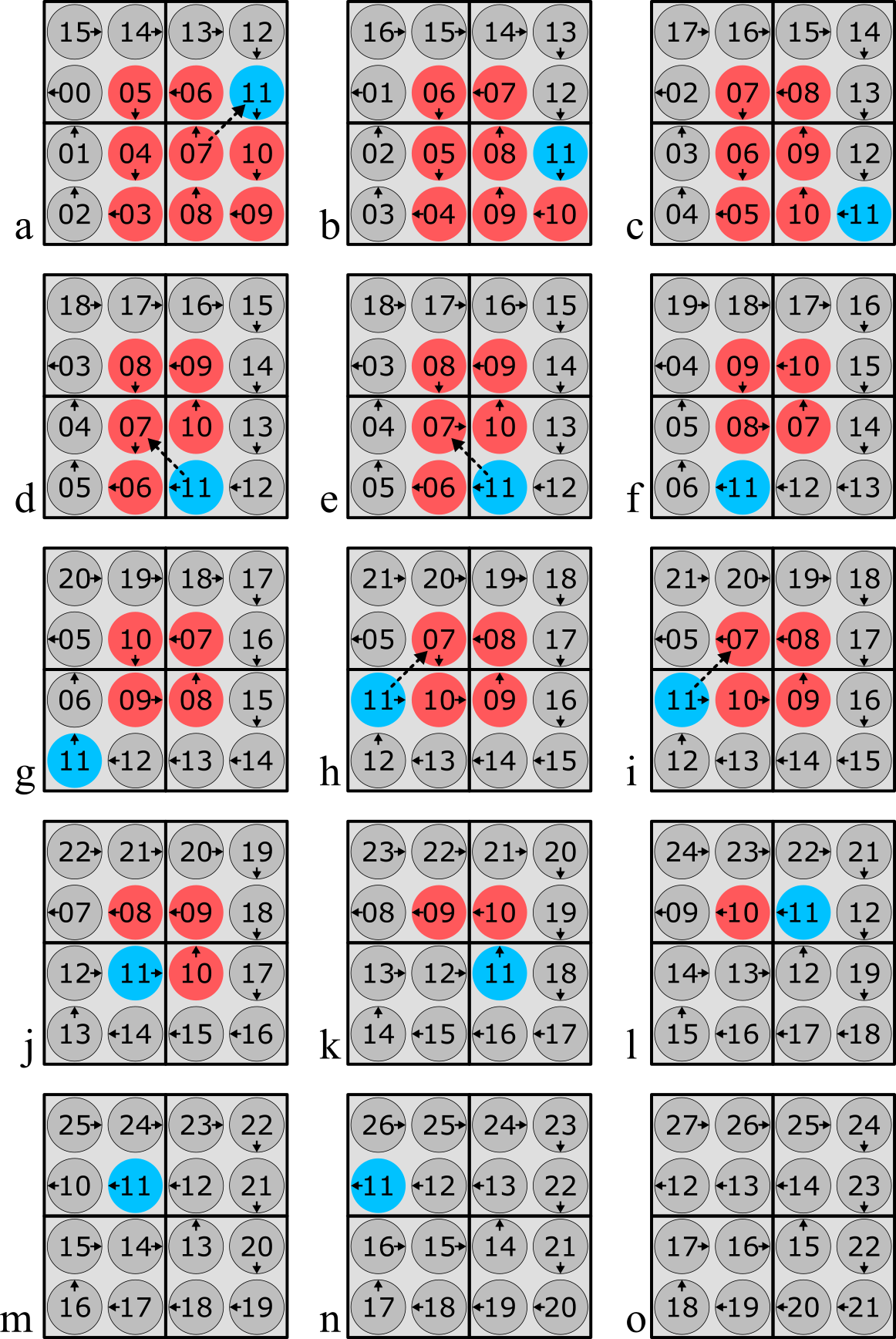}
\caption{Secondary changes via the movement-based method. The sequence transforms the interim path (a) to the preferred path (o). Numbers represent robot IDs. Arrows on the robots indicate headings. Dashed arrows indicate communication paths for direction swapping in (d), (e), (h), and (i) and promoting the change robot (a). Red robots are pass-back robots and the blue robot is the change robot.}
\label{fig:alg3_intro}
\end{figure}

An example of pass-back behavior can be seen in the transition from Fig.~\ref{fig:alg3_intro}a to~\ref{fig:alg3_intro}b.
In Fig.~\ref{fig:alg3_intro}a, robot 3 is a pass-back robot since it was adjacent to the box addition.
When it moves to its next position in Fig.~\ref{fig:alg3_intro}b, robot 3 sends a pass-back message to robot 4 and transitions back to the default behavior.
Robot 4 receives the pass-back message and adjusts its path to follow robot 3.
In Fig.~\ref{fig:alg3_intro}a, robot 5 is also a pass-back robot.
When it moves to its next position (in Fig.~\ref{fig:alg3_intro}b), it sends a pass-back message to robot 6.
Robot 5 transitions to the default behavior, but then immediately receives a pass-back message from robot 4, transitioning it back to the pass-back state.
All robots involved in the primary change (or adjacent to the primary change) become pass-back robots.
They are shown in red in Fig.~\ref{fig:alg3_intro}.

Change robot behavior is also evident in Fig.~\ref{fig:alg3_intro}.
The change robot (robot 11) is promoted by robot 7 once the local changes are complete (Fig.~\ref{fig:alg3_intro}a).
Robot 11 serves as the change robot in this case because it is the robot present at the SCSN when primary changes are completed.
Likewise, robot 7 is the robot that promotes robot 11 because robot 7 was at the point of inflection at the onset of primary changes and was the robot that led the way into the newly added box.
Robot 11 then proceeds to move through the shape, executing its default behavior and destination swapping as necessary.
For example, in Fig.~\ref{fig:alg3_intro}d, robot 11 has a ``choice" to follow the interim path ``up" to robot 10's position or to follow its default behavior ``left" to robot 6's position.
Since robot 11 always executes its default behavior, it will choose to move ``left" to the node occupied by robot 6 (see robot 11's heading in Fig.~\ref{fig:alg3_intro}d). 
Thus, robot 7, which had previously planned to move to robot 6's position, is free to swap destinations and instead move ``right" to robot 10's position (Fig.~\ref{fig:alg3_intro}e). 
Robot 7 also switches to a pass-back state as a result of the destination swap. 
This particular destination swap results in a pseudo-valid path of two sub-cycles: robots 7, 8, 9, and 10 form one sub-cycle while robots 3-6 and 11-18 form another.
This pseudo-valid path continues to persist until Fig.~\ref{fig:alg3_intro}h~and Fig.~\ref{fig:alg3_intro}i when robots 7 and 11 destination swap again.
Lastly, the change robot (robot 11) clears pass-back messages as it moves, so when robot 11 has exited the shape, the swarm is following the new preferred path (Fig.~\ref{fig:alg3_intro}o).

\section{Additional Experiments and Demonstrations}\label{sec:additional_experiments}
We performed additional experiments to demonstrate the swarm's adaptability while maintaining its persistence.
Specifically, we show detection, primary changes, and secondary changes via both the communication-based method and the movement-based method with the swarm executing the default behavior while not resolving a change.
These demonstrations were completed on a swarm of mobile robots to capture the swarm's response to an actual human and on a simulated swarm of 90 robots to depict the scalability of the algorithms.
With no known direct state-of-the-art comparisons, it was not possible to compare the performance of these demonstrations against other methods.
Instead, these demonstrations were completed to both reinforce the effectiveness of the algorithms and complement the theoretical proofs.

\subsection{Demonstrations with Humans}
Experiments were performed on the swarm of \emph{Coachbots} to show the swarm's response to human initiated changes.
These experiments used a 20 robot initial shape and a 22 robot charging station.
The demonstrations on the \emph{Coachbots} also consisted of both addition and subtraction via both secondary change methods (communication-based and movement-based).

In these experiments, once the initial shape was formed and the robots reached a steady-state persistent cycle into and out of the shape, a human initiated a change to the shape by interacting directly with the swarm. 
In both secondary change methods, and for both addition and subtraction, the swarm effectively detected each change, made primary changes, and finally made secondary changes to continue to cycle robots through the new shape.
The time to resolve each change was on the order of 1-7 minutes depending on the location of the change.
These experiments were recorded with an overhead camera, and images of one human-swarm interaction experiment are captured in Fig.~\ref{human_adapt}.

\begin{figure}[!t]
\centering
\includegraphics[width=75mm]{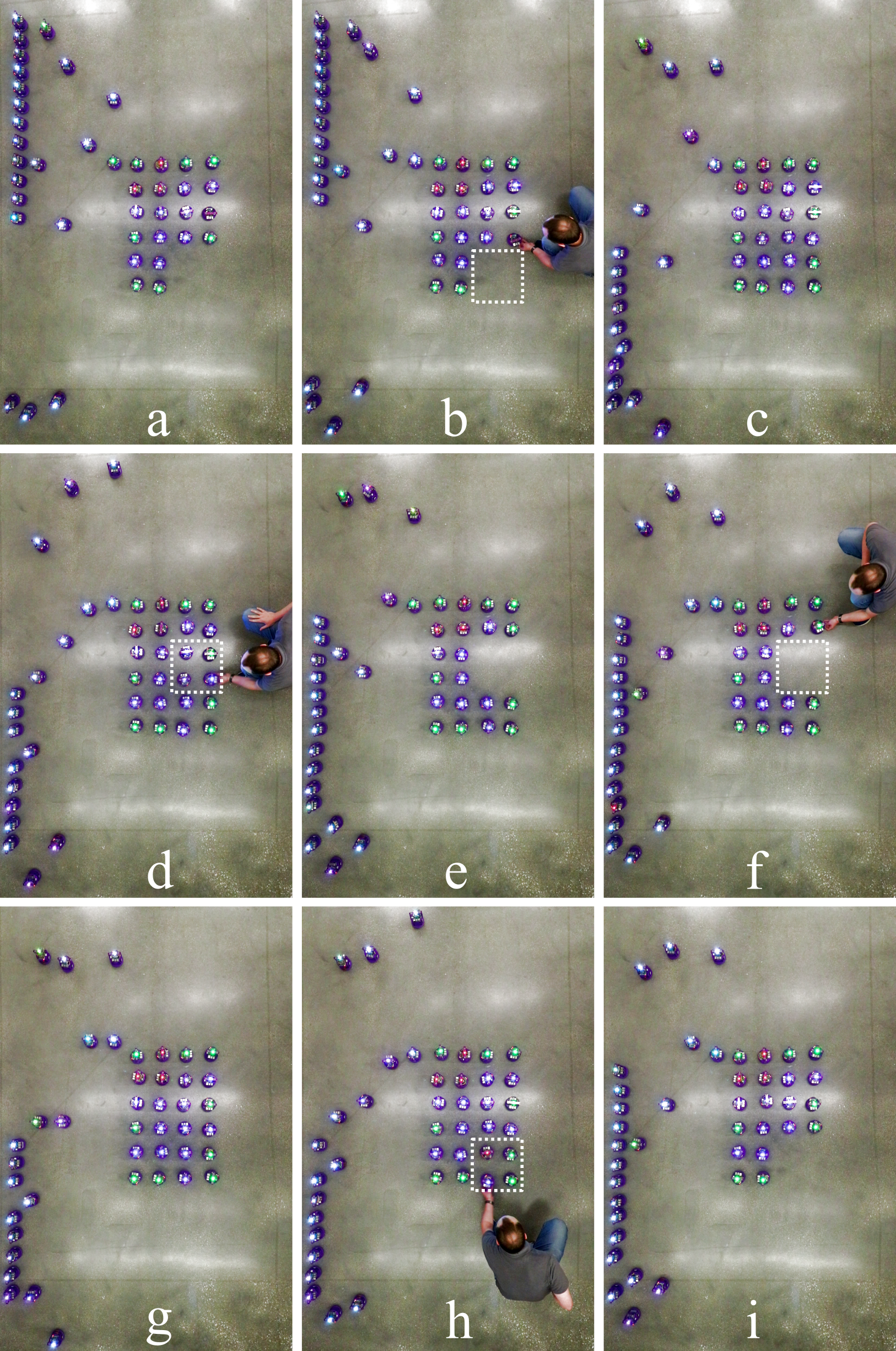}
\caption{A sequence of images from a \emph{Coachbot} experiment running the movement-based method. Pane (a) shows the initial shape. Human initiated additions are in panes (b) and (f), and subtractions are in panes (d) and (h). A dashed rectangle indicates the added or removed box. Panes (c), (e), (g), and (i) show the swarm's shape once the previous change is resolved. In-shape robots have green, blue, and red LEDs to indicate the potential for addition, subtraction, and no change, respectively.
}
\label{human_adapt}
\end{figure}

The \emph{Coachbots} were used in these experiments because they were an existing swarm available to the authors.
Unfortunately, they have no means of sensing humans: only their position $(x,y)$ and heading $(\theta)$.
Therefore, we physically rotated robots in place as human ``gestures." 
Robots changed color within the shape to indicate which robots could be manipulated for an addition and which for a subtraction.
Robots would turn green if they were along the periphery of the shape alongside a box that could be added to the shape.
Robots would turn blue if they were not along the periphery of the shape or otherwise not alongside a box that could be added to the shape.
Finally, robots would turn red if they could not be manipulated for either addition or subtraction (e.g., a robot at the exit node).
If a green robot's heading was changed, then the robot would initiate a box addition, and if a blue robot's heading was changed, then the robot would initiate a box subtraction.
Red robots were unresponsive to human interaction.
% Therefore, a proxy robot was used instead of adapting the \emph{Coachbots} or building an entirely new swarm test bed.
% Every time a human wanted to initiate a change, he or she would hold a proxy robot near the swarm (e.g., Fig.~\ref{human_adapt}b). 
% The swarm would sense the proxy robot and react accordingly.
% Since the proxy robot is entirely human manipulated, it effectively acts as an extension of the human. 
% This is similar to having a human wear special gloves, AprilTags, or other markers to make it easier for robots to sense the presence of a human in their environment.

\subsection{Large Scale Simulations}
The final set of experiments demonstrated the scalability of the algorithms by simulating a swarm of 90 robots forming a 17-box ``N" shape persistently.
% We selected 90 robots to mimic a swarm of 100 robots with 10 robots reserved as spares and not included in the task (e.g., 10 robots have been removed for maintenance).
The swarm was then manipulated into a 15-box ``U" shape through a series of additions and subtractions.
Human gestures indicating additions or subtractions were emulated by pre-programmed messages transmitted at pre-assigned times to the robots nearest to the desired addition or subtraction.
Images of a particular test are in Fig.~\ref{NU}.
This simulation demonstrates the scalability of the algorithms to swarms of large numbers and shows the complexity of the shapes that can be formed despite the valid shape constraints.
Although scalability is not formally proven, the size of the swarm is not a necessary component in any of the algorithms, and each robot is executing identical behavior without a dependency on a centralized actor.
This indicates that the size of the swarm would not be limited by the algorithms, but rather the capacity of the charging station, the size of the path that an individual robot can travel, the onboard memory of each robot, or the size of the environment.

\begin{figure}[!t]
\centering
\includegraphics[width=0.48\textwidth]{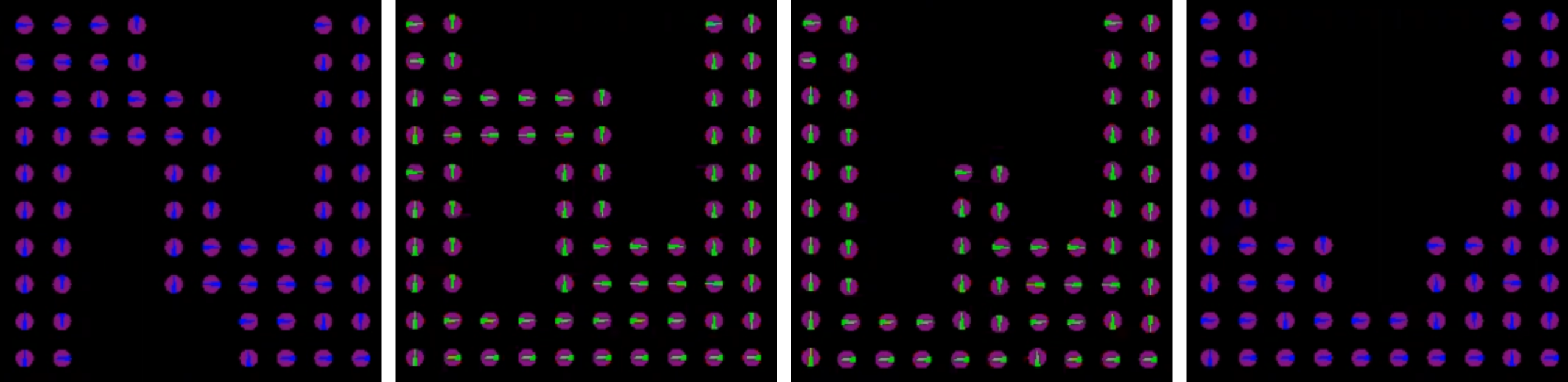}
\caption{A simulation of 90 robots running the communication-based method.  Robots begin in a 17-box ``N" (far left). Boxes are removed and added to the shape (left to right) until the swarm forms the 15-box ``U" (far right). Robot LED color indicates shapes of with an even (green) or odd (blue) number of boxes. The charging station is not shown.
}
\label{NU}
\end{figure}

\section{Shape Adaptability Theory}\label{sec:adaptability_theory}
The previous sections described swarm adaptability in the cases of box addition and box subtraction. 
Specifically, we covered the steps of detection, primary changes (for both addition and subtraction), and secondary changes (for both the communication-based method and the movement-based method).
Along the way, we stated a series of unproven claims as fact.
The remainder of this section formally addresses these claims and provides proofs to support their validity.

\subsection{Primary Changes: Addition Yields Valid Paths}\label{sec:addition_theory}
Earlier we claimed that, in the event of an addition, primary changes always result in a valid interim path. 
We can prove this claim directly by the logic of Proof~\ref{claim_random}.

\begin{theorem}\label{claim_add_valid}
    Addition primary changes result in a valid path.
\end{theorem}
\begin{proof}[\proofname\ 3]\label{proof:add_is_valid}\noindent
Prior to an addition, the existing shape is a valid shape and the existing path is a preferred path.
In the event of an addition, primary changes effectively replace a periphery edge of the existing path with a 5-edge loop: 1 edge that spans into the new box from an existing box, 3 non-spanning edges that move robots clockwise through the new box, and 1 spanning edge back into the existing box.
This is the same 5-edge loop discussed in Proof~\ref{claim_random} that occurs when a unit path through a new box is merged into the existing path.
Thus, by the same logic as Proof~\ref{claim_random}, we know that the 5-edge loop created by primary changes in response to an addition will also result a valid path and Theorem~\ref{claim_add_valid} is valid.
\end{proof}

\subsection{Primary Changes: Subtraction Yields Pseudo-valid Paths}\label{sec:subtraction_theory}
Earlier we claimed that local changes in response to a subtraction can only guarantee that the resulting path is at least pseudo-valid.
In order to prove this, we must first define what a pseudo-valid path is and explain the phrase ``\emph{at least pseudo-valid}."
We will also establish two more lemmas.

\begin{lemma}\label{lem:cw}
    For a path that is pseudo-valid, valid, or preferred, each non-spanning edge results in clockwise motion around the center of its box.
\end{lemma}

\begin{lemma}\label{lem:pairs}
    For a path that is pseudo-valid, valid, or preferred, each spanning edge belongs to a pair of parallel and opposite spanning edges that cross over the same box connection.
\end{lemma}

We define a \textbf{pseudo-valid path} as any discontinuous path created by separating a valid path into sub-cycles where each \textbf{sub-cycle} is a planar Hamiltonian cycle through a portion of the shape.
Valid paths can be broken into pseudo-valid paths via path separation (Fig.~\ref{fig:operations}) by Lemma~\ref{lem:separation}.
Valid paths can also be broken into pseudo-valid paths via a process similar to path merging (replacing a pair of parallel and opposite non-spanning edges in the existing valid path with a pair of parallel and opposite spanning edges).
For example, the valid path in Fig.~\ref{fig:movement_overview}c could be turned into the pseudo-valid path in Fig.~\ref{fig:movement_overview}b by replacing the red pair of non-spanning edges in Fig.~\ref{fig:movement_overview}c with a pair of spanning edges.

Since pseudo-valid paths are constructed by breaking a valid path into sub-cycles, each psudo-valid path maintains some properties of a valid path.
Specifically, each non-spanning edge results in clockwise motion around the center of its box, and each spanning edge belongs to a pair of parallel and opposite spanning edges that cross over the same box connection (though the spanning edges may belong to different sub-cycles).
Since these properties are true of preferred, valid, and pseudo-valid paths, we can conclude Lemma~\ref{lem:cw} and Lemma~\ref{lem:pairs}.
Examples of a preferred path, valid path, and pseudo-valid path are shown in Fig.~\ref{venn} for comparison.

\begin{figure}[!t]
\centering
\includegraphics[width=80mm]{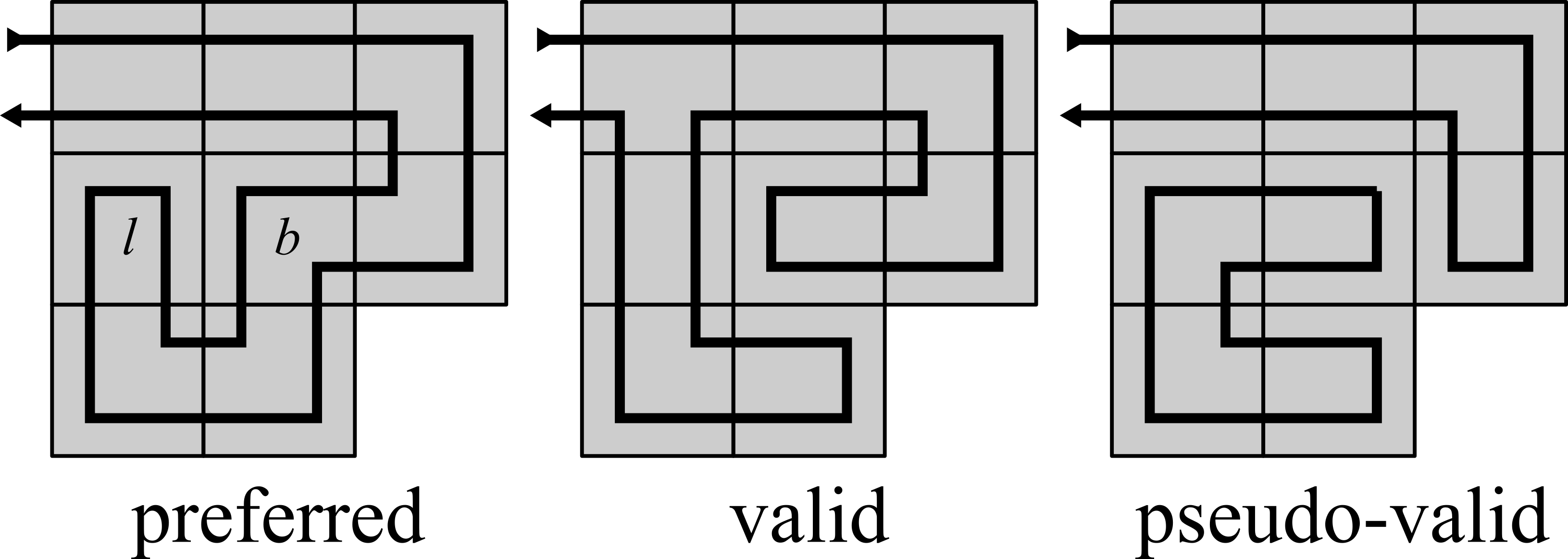}
\caption{Examples of preferred, valid, and pseudo-valid paths through the same shape. Arrows indicate path direction.}
\label{venn}
\end{figure}

Since preferred paths are a subset of valid paths, and valid paths are a subset of pseudo-valid paths (i.e., valid paths are pseudo-valid paths with only one sub-cycle), then we can use the phrase ``\emph{at least a pseudo-valid path}" to refer to any path that is pseudo-valid but could also be valid or preferred. %This is done to simplify the language of the remainder of this paper.
With these terms defined, we can now prove our claim directly using the concepts of path separation, path merging, and Lemma~\ref{lem:separation}.

\begin{theorem}\label{claim_sub_pseudovalid}\noindent
    Subtraction primary changes result in at least a pseudo-valid path.
\end{theorem}
\begin{proof}[\proofname\ 4]
Prior to a subtraction, the existing shape is a valid shape and the existing path is a preferred path.
Since the existing path is a preferred path, by Proof~\ref{claim_default} we know that a robot traversing the preferred path will create a tree data structure of boxes that matches the tree data structure assembled by the DFCP method (just like in Fig.~\ref{fig:cwp}). 
The box with entry and exit nodes is the root of the tree, and the remainder of the boxes in the shape are either branches or leaves.
Branches are boxes that the robot enters and exits at least twice: at least once to access descendent boxes and once to move back up the tree to exit the shape. 
On the other hand, leaves are boxes for which the robot enters and exits only once, making a distinctive 5-edge ``$\cup$" pattern through the box.
For example, in Fig.~\ref{venn}, the box marked ``\textit{l}" is a leaf and the box marked ``\textit{b}" is a branch.
Now, let $R$ denote a box to be removed.
The impact of removing $R$ from a valid shape (with a preferred path) varies depending on if $R$ is a leaf or a branch. 

If $R$ is a leaf, then the box can be removed without impacting any other portions of the path because the original path did not need to travel through $R$ to get to other boxes. 
When $R$ is removed, the robot that had previously planned to take a spanning edge into $R$ will instead plan to move clockwise within its own box (in accordance with primary change procedures).
In that case, removing $R$ is like separating its unit path from the planar Hamiltonian cycle in the remainder of the shape via path separation (i.e., spanning edges become non-spanning edges).
This results in a planar Hamiltonian cycle through the remainder of the shape (by Lemma~\ref{lem:separation}), and the path prior to and after $R$ is not affected by the presence of $R$ (by Lemmas~\ref{lem:influence_up}~and~\ref{lem:influence_down}).
Thus, we can conclude that the planar Hamiltonian cycle that remains once $R$ is removed is identical to the planar Hamiltonian cycle that would have existed if $R$ had never been a part of the shape.
We can also conclude that the planar Hamiltonian cycle that would have existed if $R$ had never been a part of the shape is the same as the preferred path that a robot would have traveled via the default behavior if $R$ had never been a part of the shape as $R$ (being a leaf) is the last unvisited box in its subtree.
In other words, we can conclude that when a leaf box is removed, primary changes result in the preferred path of the new shape.

However, if $R$ is a branch, then removing the box will result in a discontinuity in the path since the path had to pass through $R$ to get to other boxes. 
In that case, all robots that had previously planned to take a spanning edge into $R$ will instead plan to move clockwise within their own boxes (in accordance with primary change procedures).
This is the same as a path separation operation on each pair of spanning edges that crossed a side of $R$.
By Lemma~\ref{lem:separation}, this still creates planar Hamiltonian cycles in the remainder of the boxes, but instead of being one continuous path, the resulting path will have at least two sub-cycles. 
For example, if $R$ is a branch with only one child, there will be one sub-cycle through all of the boxes visited from the root box through the parent of $R$ and one sub-cycle through all of boxes visited after $R$.
By definition, such a path is a pseudo-valid path, so we can conclude that when a branch is removed, primary changes result in a pseudo-valid interim path.

Finally, since a primary changes in response to a subtraction result in either a preferred path (in the case of a leaf) or a pseudo-valid path (in the case of a branch), we can conclude that Theorem~\ref{claim_sub_pseudovalid} is valid.
\end{proof}

\subsection{Secondary Changes Yield Preferred Paths}\label{sec:secondary_changes_proof}
In describing secondary changes, we presented two interchangeable methods for operation: the communication-based method and the movement-based method.
We also explained that secondary changes return the swarm to the preferred path of the new shape so that robots can freely execute the default behavior without concerning themselves about shape changes that occurred in the past.
Therefore, we must prove that both the communication-based method and the movement-based method result in the new preferred path of the new shape.
The proof of the communication-based method is trivial.
Since the communication-based method is just a virtual implementation of the default behavior, we can conclude that it will always result in the new preferred path of the new shape by Proof~\ref{claim_default}. 

However, the proof of the movement-based method is less obvious.
We earlier claimed that by repeatedly destination swapping with the change robot, the swarm will morph through a sequence of interim paths until the preferred path is achieved.
This relies on two basic tenets: 1) that a destination swap morphs the swarm from one path that is at least pseudo-valid to another path that is at least pseudo-valid, and 2) that the sequence of interim paths will converge to the preferred path.
Both of these are addressed in Proof~\ref{claim_swap}.

\begin{theorem}\label{claim_swap}\noindent
    Destination swapping results in a sequence of valid or pseudo-valid paths until the new preferred path is achieved.
\end{theorem}
\begin{proof}[\proofname\ 5]
First, by Proof~\ref{claim_add_valid} and Proof~\ref{claim_sub_pseudovalid}, it can be said that all local changes, regardless of whether a box was added or removed, will result in at least a pseudo-valid path (i.e., the path is either pseudo-valid, valid, or preferred).
Thus, it can be assumed that when a change robot is promoted at the start of secondary changes, the interim path is at least pseudo-valid.
Second, by definition, a destination swap occurs when the change robot has a ``choice" between an edge in the interim path and an edge in the preferred path.
We can break the possible choices down into three types: a choice between two non-spanning edges, a choice between two spanning edges, or a choice between a spanning and a non-spanning edge.
% We will ignore the case where the interim path is the preferred path since the change robot will never destination swap as it will never vary from the interim path.
% Thus, we will focus on the cases where the interim path is either pseudo-valid or valid.

We can say that the change robot will never choose between two non-spanning edges because for any given node in a valid shape, there is only one outgoing non-spanning edge that results in clockwise motion around the center of its box. 
Since the interim path is at least pseudo-valid, we know that if the interim path edge leaving the change robot's node is non-spanning, then it will result in clockwise motion around the center of its box (by Lemma~\ref{lem:cw}).
Likewise, if the default behavior (preferred path) edge leaving the change robot's node is non-spanning, then it will also result in clockwise motion around the center of its box (by Lemma~\ref{lem:cw}).
Since only one non-spanning edge can produce clockwise motion around the center of its box for each node in a valid shape, the interim path edge and the preferred path edge must be the same.
Thus, we can conclude that a change robot will never choose between two non-spanning edges, so a destination swap will never occur in this case.

We can also say that the change robot will never choose between two spanning edges given Lemma~\ref{lem:cw} and Lemma~\ref{lem:pairs}.
For any given node in a valid shape with at least a pseudo-valid path, there is only one possible spanning edge that a robot can take to move away from that node while maintaining parallel and opposite pairs of spanning edges (Lemma~\ref{lem:pairs}) and clockwise motion for non-spanning edges (Lemma~\ref{lem:cw}).
Fig.~\ref{no_span_span} shows a visual representation of this fact for one node in an arbitrary box in an arbitrary valid shape (the argument is rotationally symmetric for all other nodes in the box).
Of the two possible options for spanning edges leading away from the white node in Fig.~\ref{no_span_span}, only edge $a$ is compatible for paths that are at least pseudo-valid.
If edge $b$ were present, then edge $c$ must also exist by Lemma~\ref{lem:pairs}.
However, the only way a path could connect from the end of edge $b$ to the start of edge $c$ without intersection is for the path to have a non-spanning edge with \emph{counter-clockwise} motion around the center of its box (if not in the box that edge $b$ spans into then in some other box downstream of edge $b$). 
Since this contradicts Lemma~\ref{lem:cw}, we can infer that edge $b$ will never exist in a pseudo-valid, valid, or preferred path.
Thus, the interim path and the preferred path edge will both be edge $a$, and the change robot would not be faced with a ``choice," so a destination swap would never occur in this case.

\begin{figure}[!t]
\centering
\includegraphics[width=38mm]{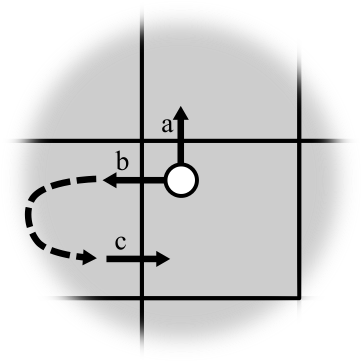}
\caption{Edge $a$ is the only plausible spanning edge to leave the white node in a path that is at least pseudo-valid. Edges $b$ and $c$ cannot exist without a counter-clockwise spanning edge somewhere in the shape (indicated by the dashed line).
}
\label{no_span_span}
\end{figure}

Therefore, to prove that destination swapping will result in a sequence of valid or pseudo-valid paths, one only has to consider two cases.
The first case is the ``spanning to non-spanning" case where the change robot rejects the interim path's spanning edge in favor of a non-spanning edge in accordance with its default behavior.
The second case is the opposite: ``non-spanning to spanning" where the change robot rejects the interim path's non-spanning edge in favor of a spanning edge in accordance with its default behavior.

Consider the first case, drawn generally for any given valid shape in Fig.~\ref{span_non_span}~(left). 
For a change robot positioned at the white node, the interim path edge is the spanning edge (edge $a$) and the default behavior path edge is the non-spanning edge (edge $b$).
Given edge $a$ is in the interim path, then edge $c$ must also exist in the interim path to form a pair of parallel and opposite spanning edges (by Lemma~\ref{lem:pairs}). 
Based on the definitions of valid and pseudo-valid paths, each node is connected by two edges: one incoming and one outgoing. 
Thus, since edge $c$ is already an outgoing edge from the black square node, then edge $d$ must not exist in the interim path.
Therefore, swapping from edge $a$ to edge $b$ for the change robot and edge $c$ to edge $d$ for the robot at the square node is possible without interrupting any other edges in the interim path and ensuring that each node still has one incoming and one outgoing edge.
% This is the same process as a path separation, so we can conclude by Lemma~\ref{lem:separation} that the path remains at least a pseudo-valid path for the ``spanning to non-spanning" case.
Furthermore, all aspects of valid and pseudo-valid paths are maintained because there are no intersections, each node is still connected via exactly two edges, the start and end nodes have not been changed, and no periphery edges have been changed (two non-periphery edges are replaced by two non-periphery edges).
The only impact of swapping destinations is that the nodes may belong to a different sub-cycle (or valid path) than they previously belonged to.
Thus, we can conclude that the path remains at least pseudo-valid for the ``spanning to non-spanning" case.

The other case is similar, and it can be seen in Fig.~\ref{span_non_span}~(right).
For a change robot positioned at the white node, the interim path edge is the non-spanning edge (edge $a$) and the default behavior path edge is the spanning edge (edge $b$).
Given edge $b$ is not in the interim path (by definition), then neither is edge $d$ (since spanning edges come in pairs: Lemma~\ref{lem:pairs}).
Since edges $b$ and $d$ do not exist in the interim path, there cannot be any spanning edges across the box connection, and the two upper nodes in Fig.~\ref{span_non_span}~(right) must be connected via a non-spanning edge. 
Ergo, edge $c$ must exist in the interim path by Lemma~\ref{lem:cw}.
So, as before, swapping from edge $a$ to edge $b$ for the change robot and edge $c$ to edge $d$ for the robot at the square node is possible without interrupting any other edges in the interim path and ensuring that each node still has one incoming and one outgoing edge.
% This is the same process as path merging, so we can conclude by Lemma~\ref{lem:merging} that the path remains at least a pseudo-valid path for the ``non-spanning to spanning" case.
This swap maintains all aspects of valid and pseudo-valid paths, so the only impact of the destination swap would be if any nodes have changed from one sub-cycle (or valid path) to another.
Thus, we can conclude that the path remains at least a pseudo-valid path for the ``non-spanning to spanning" case.
And, since this is true for both cases, we can conclude that a sequence of destination swaps will result in a sequence of valid or pseudo-valid paths.

\begin{figure}[t]
\centering
\includegraphics[width=80mm]{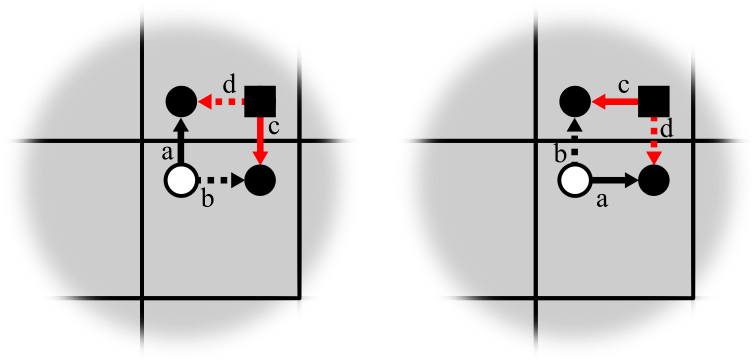}
\caption{The two cases of destination swapping. Change robots are white circles. Black squares are robots that destination swap in response to the change robot. Solid arrows indicate existing interim path edges. Dashed arrows indicate edges post destination swap. (left) Change robot chooses a non-spanning edge instead of a spanning edge. (right) Change robot chooses a spanning edge instead of a non-panning edge.
}
\label{span_non_span}
\end{figure}

Finally, as the change robot moves, it executes its default behavior, and so do all of the robots upstream of its position.
This implies that, as the change robot moves closer to the end of the shape, more robots are following the new preferred path via the default behavior and fewer robots are following the interim path.
Therefore, a sequence of valid and pseudo-valid paths created as a change robot moves through the shape will converge to the new preferred path as the change robot converges toward the exit node, and Theorem~\ref{claim_swap} is valid.
\end{proof}

\section{Conclusion}\label{sec:conclusion}
%Main point
In this work, we presented algorithms for persistent and adaptive 2D shape formation that allow a swarm to overcome the limitations of individual robot power constraints and the inflexibility of predefined shapes. 
We have shown that these algorithms are provably correct and have demonstrated their effectiveness in both simulation and a swarm of physical robots.
%New significance or application
More significantly, though, we have opened the door to a largely unexplored field of swarm shape formation applications where the duration of the task may no longer be a constraint on the swarm's performance and the shape that the swarm is forming during its task is free to change in response to an external stimulus such as a gesture by a human.

%Call for more research
%Furthermore, there is still plenty of room for future work in this area. 
Future work includes demonstrating these algorithms on a swarm of physical flying robots. % is self evident and has already been discussed.
There is also work to be done to improve the efficiency and fault tolerance of these algorithms to make them more suitable for practical application.
Finally, we intend to develop three-dimensional algorithms that are analogous to the 2D algorithms presented in this paper so that swarm implementations are not constrained to planar tasks.
One way to accomplish this might be to execute 2D algorithms for a set of vertically stacked layers in a 3D shape. 
Another method might be to develop valid 3D shapes constructed of cubes in the same way we developed valid shapes constructed of boxes in 2D.
Additional challenges with a 3D version will include the impact of downwash on other flying robots, the ability to detect human gestures while in flight, and the safety concerns associated with a human interacting with flying robots.

\printbibliography
\onecolumn
\appendix

\subsection{Defined Terms}\label{app:terms}
% \begin{table}[t]
%     \centering
    
%     % \caption{Caption}
%     \label{tab:definitions}
% \end{table}
\begin{tabular}{l c p{.7\textwidth}}
\toprule
Term       &                  & Definition \\
\midrule
\textbf{box}             &$:$& four grid nodes in a 2x2 square \\
\textbf{change robot}    &$:$& the robot whose movement initiates destination swapping in the movement-based method \\
\textbf{communication-} && \\
\quad\textbf{based method}             &$:$& a method for secondary changes that converts an interim path to the new preferred path by passing a memory message \\
\textbf{default behavior}      &$:$& algorithm used for shape persistence; results in a preferred path \\
\textbf{depth-first} && \\
\quad\textbf{clockwise-priority} && \\
\quad\textbf{(DFCP) method} &$:$& a method for assembling valid shapes by assembling boxes in accordance with a traditional depth-first search where ties are resolved with a clockwise priority \\
\textbf{destination swap}             &$:$& process by which a neighboring robot assumes the interim path destination of the change robot instead of its own interim path destination \\
\textbf{DFCP path}     &$:$& a type of valid path generated by the DFCP method\\
\textbf{downstream}     &$:$& refers to a robot that has previously visited a particular grid node\\
\textbf{existing path}     &$:$& the path prior to a change (i.e., addition or subtraction)\\
\textbf{existing shape}     &$:$& the shape prior to a change (i.e., addition or subtraction)\\
\textbf{interim path}     &$:$& a provisional path formed in the process of changing from an existing path to a new path\\
\textbf{memory message}     &$:$& the message used to communicate change in the communication-based method\\
\textbf{movement-based} && \\
\quad\textbf{method}             &$:$& a method for secondary changes that converts an interim path to the new preferred path by a series of destination swaps initiated by the movement of a change robot \\
\textbf{new path}     &$:$& the path after a change is resolved\\
\textbf{new shape}     &$:$& the shape after a change is resolved\\
\textbf{non-spanning edge}     &$:$& an edge that begins and terminates in the same box\\
\textbf{opposite edges}     &$:$& edges that have directions exactly 180\textdegree{} from one another\\
\textbf{pair}     &$:$& (e.g., a \emph{pair} of edges) two edges with start and end nodes in a 2x2 grid configuration\\
\textbf{parallel edges}     &$:$& edges that are equidistant everywhere and do not intersect\\
\textbf{pass-back message}   &$:$& message used to communicate non-default movement during secondary changes\\
\textbf{pass-back robots}   &$:$& robots that send pass-back messages to upstream robots\\
\textbf{path merging}     &$:$& the process of replacing a pair of parallel and opposite non-spanning edges with a pair of parallel and opposite spanning edges\\
\textbf{path separation}     &$:$& the process of replacing a pair of parallel and opposite spanning edges with a pair of parallel and opposite non-spanning edges\\
\textbf{periphery edge}     &$:$& a non-spanning edge between two periphery nodes\\
\textbf{periphery node}     &$:$& an in-shape node along the periphery of the shape\\
\textbf{preferred path}     &$:$& a type of valid path generated by the default behavior\\
\textbf{pseudo-valid path}     &$:$& a discontinuous path created by separating a valid path into sub-cycles\\
\textbf{secondary change} && \\
\quad\textbf{start node (SCSN)}             &$:$& the node from which secondary changes begin \\
\textbf{spanning edge}     &$:$& an edge that begins and terminates in different boxes\\
\textbf{sub-cycle}     &$:$& a planar Hamiltonian cycle through a portion of the shape\\
\textbf{unit path}     &$:$& a clockwise Hamiltonian cycle through a unit shape\\
\textbf{unit shape}     &$:$& a shape constructed of 1 box\\
\textbf{upstream}     &$:$& refers to a robot that has yet to visit a particular grid node\\
\textbf{valid path}     &$:$& a planar Hamiltonian cycle with adjacent entry and exit nodes on the periphery of the shape such that each non-spanning edge is directed clockwise around the center of its box and each spanning edge is part of a pair of parallel and opposite edges spanning across the same box connection\\
\textbf{valid shape}     &$:$& a shape constructed of boxes assembled continuously such that each box meets the full side of another box\\
\bottomrule
\end{tabular}

\end{document}